\documentclass[sn-mathphys,Numbered]{sn-jnl}%

\usepackage{graphicx}%
\usepackage{multirow}%
\usepackage{amsmath,amssymb,amsfonts}%
\usepackage{amsthm}%
\usepackage{mathrsfs}%
\usepackage[title]{appendix}%
\usepackage{xcolor}%
\usepackage{textcomp}%
\usepackage{manyfoot}%
\usepackage{booktabs}%
\usepackage{algorithm}%
\usepackage{algorithmicx}%
\usepackage{algpseudocode}%
\usepackage{listings}%
\usepackage{import}
\usepackage{albi_pack}

\theoremstyle{thmstyleone}%
\newtheorem{theorem}{Theorem}%
\newtheorem{proposition}[theorem]{Proposition}%

\theoremstyle{thmstyletwo}%
\newtheorem{example}{Example}%
\newtheorem{remark}{Remark}%

\theoremstyle{thmstylethree}%
\newtheorem{definition}{Definition}%

\usepackage{color,array}

\usepackage{graphicx}

\usepackage{changes}
\definecolor{added}{rgb}{0,0.5,0} 
\setdeletedmarkup{\textcolor{red}{\sout{#1}}} 
\setaddedmarkup{\textcolor{added}{\protect{\uline{#1}}}}

\renewcommand{\sout}[1]{}
\setdeletedmarkup{\sout{#1}}
\definecolor{added}{rgb}{0,0,0}

\renewcommand{\uline}[1]{#1}

\setaddedmarkup{\textcolor{added}{\protect{#1}}}

\begin{document}

\title[Conformal Prediction and Scalable Classifiers]{Conformal Predictions for Probabilistically Robust \\ Scalable Machine Learning Classification }


\author*[1,4]{\fnm{Alberto} \sur{Carlevaro}}\email{alberto.carlevaro@ieiit.cnr.it}

\author[3]{\fnm{Teodoro} \sur{Alamo}}\email{talamo@us.es}

\author[2]{\fnm{Fabrizio} \sur{Dabbene}}\email{fabrizio.dabbene@cnr.it}

\author[1]{\fnm{Maurizio} \sur{Mongelli}}\email{maurizio.mongelli@ieiit.cnr.it}

\affil*[1]{\orgdiv{ Istituto di Elettronica e di Ingegneria
dell’Informazione e delle Telecomunicazioni}, \orgname{CNR}, \orgaddress{\street{Corso Ferdinando Maria Perrone, 24}, \city{Genoa}, \postcode{16152}, \country{Italy}}}

\affil[2]{\orgdiv{ Istituto di Elettronica e di Ingegneria
dell’Informazione e delle Telecomunicazioni}, \orgname{CNR}, \orgaddress{\street{Corso Duca degli Abruzzi, 24}, \city{Turin}, \postcode{10129}, \country{Italy}}}

\affil[3]{\orgdiv{Departamento de Ingenierìa
de Sistemas y Automàtica}, \orgname{Universidad de Sevilla, Escuela Superior de
Ingenieros}, \orgaddress{\street{Camino de los Descubrimientos}, \city{Seville}, \postcode{41092}, \state{Spain}}}

\affil[4]{\added{\orgdiv{Funded Research Department}, \orgname{Aitek SpA}, \orgaddress{\street{Via della Crocetta 15}, \city{Genoa}, \postcode{16122}, \state{Italy}}}}


\abstract{ Conformal predictions make it possible to define reliable and robust learning algorithms. But they are essentially a method for evaluating whether an algorithm is good enough to be used in practice. To define a reliable learning framework for classification from the very beginning of its design, the concept of scalable classifier was introduced to generalize the concept of classical classifier by linking it to statistical order theory and probabilistic learning theory. In this paper, we analyze the similarities between scalable classifiers and conformal predictions by introducing a new definition of a score function and defining a special set of input variables, the conformal safety set, which can identify patterns in the input space that satisfy the error coverage guarantee, i.e., that the probability of observing the wrong (possibly unsafe) label for points belonging to this set is bounded by a predefined $\varepsilon$ error level. We demonstrate the practical implications of this framework through an application in cybersecurity for identifying DNS tunneling attacks. Our work contributes to the development of probabilistically robust and reliable machine learning models.}




\maketitle

\section{Introduction}\label{sec1}

\subsection{Context}
\added{Conformal} predictions (CPs) \citep{JMLR:v9:shafer08a} are gaining increasing importance in machine learning (ML) since they validate algorithms in terms of confidence of the prediction. Although it is a fairly recent field of study, there has been an astonishing production of scholarly papers, from the definition of new score functions to different methodologies for constructing conformal sets and, of course, a wide variety of applications. In fact, the ferment of scientific research in this field is so active that even the father of this theory, V. Vovk\footnote{V.\ Vovk. is the author of the groundbreaking book \emph{Algorithmic Learning in a Random World}, \citep{vovk2005algorithmic}, which is the foundation of the theory of conformal prediction. His scientific production is still at the top of research in the field, and a constantly updated list of papers on CP by Vovk and his colleagues can be found here \url{http://www.alrw.net/}. }, continues to actively contribute to the improvement of its knowledge, as in the case of \cite{vovk2017nonparametric} where he and his colleagues investigate the concept of validity under nonparametric hypotheses or the innovative introduction of Venn predictors as in \cite{vovk2022probabilistic}.
We refer the reader to the surveys \cite{gentleIntro,fontana2020conformal,toccaceli2022introduction} that largely cover all recent publications and discussions on uncertainty quantification (UQ) through CP for machine learning models.\\
Under canonical CP theory, the definition of a score function is very peculiar to either the classifier or the application at hand. For example, \cite{forreryd2018predicting} defines a special conformity measure (corresponding to a score function), based on the residual between the calibration points and the classification hyperplane of a SVM model. \added{Other example, always SVM-based, can be found in \cite{forreryd2018predicting}, \cite{JMLR:v9:shafer08a} and \cite{balasubramanian2009support}, where different definitions of score function (or conformity/non-conformity measure) are given. One of the strengths of our approach, as will become clear later, is the unique definition of such a score function, which, given \emph{any} classifier, allows the conformal prediction framework to be applied in the most natural way.} The work in \cite{narteni2023confiderai} defines a score function for rule-based models. The softmax function is used as score function in most image classification problems as in \cite{angelopoulos-sets,park2019pac,andeol2023conformal}, and many other examples can be provided (see, e.g., the above cited surveys). As a matter of the fact, those definitions come after the setting of the classifier and do not outline a common methodology. 

\subsection{Contribution}
By focusing on binary classification, our goal is to introduce to the CP community a way to link ML classifiers with a natural definition of score function \deleted{to embed}\added{that embeds} the conformal guarantee by construction.  \\

We exploit the concept of \emph{scalable classifiers} $f_\thB(\x,\rho)$ (Section \ref{sec:scalable}) introduced in \cite{carlevaro2023probabilistic} to develop a new class of score functions that rely on the geometry of the problem and that are naturally built from the classifier itself, by inheriting its properties (Section \ref{sec:score}). This allows CP theory to derive the relationship between the input space and the conformity guarantee explicitly. By introducing the new concept of \emph{conformal safety region}, we provide an analytical form of the specific subsets of the input space in which marginal coverage guarantees on prediction (Section \ref{sec:csr}) can be ensured. \deleted{The control of}\added{Controlling} the misclassification rate (either false positives or false negatives) naturally \deleted{arises, by posing}\added{follows from eliciting} the following quantities: the confidence level given by the conformal framework, the binary output $y\in\{+1,-1\}$, the confidence error $\varepsilon\in(0,1)$\added{,} as well as the new notion of conformal safety set $\mathcal{S}_\varepsilon$ that satisfies 
$$\Pr\{y = -1 \ \textrm{and} \ \x \in\mathcal{S}_\varepsilon\}\le \varepsilon.$$
In short, the paper defines a methodology in which the optimal shape of a classifier is derived, where the optimality criterion is embedded in the classifier by the conformal guarantee. \added{The proposed methodology thus  places itself in the recent and as yet unexplored field of set-value classification \cite{chzhen2021set}, a broad theory that studies predictors that have both good prediction properties and specific performance requirements, two points that underlie the proposed research.}\\

The remainder of the article is organized by providing a brief recall of the concepts of scalable classifiers and conformal prediction and then delving into the details of the definition of scalable score function and conformal safety region. The whole procedure is then validated on an application use case related to cyber-security for identifying DNS tunneling attacks (Section \ref{sec:DNS}).

\section{Background: Scalable Classifiers and Conformal Prediction}

The background of the theory we would like to propose in this research refers to a new interpretation of classical classification algorithms, scalable classifiers, and another rather new theory on \deleted{trusted}\added{trustworthy} AI, called conformal prediction. Both of these techniques belong to the field of reliable AI, searching for the definition of models, procedures or bounds that can make a learning algorithm probabilistically robust and reliable.

\subsection{Scalable Classifiers}
\label{sec:scalable}
\added{Given an input space $\mathcal{X}\subseteq\R^d$, $d\in\N^+$, and an output space $\mathcal{Y} = \{-1,+1\}$}, scalable classifiers (SCs) were introduced in \cite{carlevaro2023probabilistic} as a family of (binary) classifiers parameterized by a scale factor $\rho\in\R$
\begin{equation}
\label{eq:phi:minus}
    \phi_{\thB}(\x,\rho)  \doteq 
    \begin{cases}
        +1 \quad \quad \text{if } \, f_\thB(\x,\rho) < 0, \\
        -1 \quad \quad \text{otherwise.}
    \end{cases}
\end{equation}
where the function $f_\thB: \mathcal{X}\times\R \longrightarrow \R$ is the so-called \textit{classifier predictor} and the notation with subscript $\thB$ refers to the fact that the classifier also depends on a set of \emph{hyperparameters} $\thB=[\thB_1,\cdots,\thB_{n_{\thB}}]^\top$ to be set in the model (e.g.\ different choices of kernel,  regularization parameters, etc.). To give a meaningful interpretation of this classifier, we refer to the class $+1$ as a ``safe'' situation we want to target and the other class with $-1$ as an ``unsafe'' situation. Some examples might be differentiating between a patient's condition in developing or not developing a certain disease \citep{10.1371/journal.pone.0272825}, or understanding what input parameters lead an autonomous car to a collision or non-collision \citep{9787552}, among many other applications.  \\
SCs rely on the main assumption that for every $\x\in\mathcal{X}$, $f_\thB(\x,\rho)$ is continuous and monotonically increasing \added{in $\rho$}\added{,} and that $\lim\limits_{\rho\to -\infty}f_{\thB}(\x,\rho)<0<\lim\limits_{\rho\to \infty} f_{\thB}(\x,\rho)$, \cite[Assumption 1]{carlevaro2023probabilistic}. These assumptions imply that, 
there exists a unique solution $\bar{\rho}(\x)$ to the equation 
\begin{equation}
 f_\thB(\x,\textrm{\deleted{$\rho(\x)$}}\added{\rho})=0.
 \label{eq:barrho}
\end{equation}
The proof of this claim is available in \cite[Property 2]{carlevaro2023probabilistic}. 
In words, a scalable classifier is a classifier that satisfies some crucial properties: $i)$ given $\x$, there is always a value of $\rho$, denoted $\bar{\rho}(\x)$, that establishes the border between the two classes, $ii)$ the increase of $\rho$ forces the classifier to predict the $-1$ class and $iii)$ the target $+1$ class of a given feature vector $\x$ is maintained for a decrease of $\rho$. Moreover, \cite[Property 3]{carlevaro2023probabilistic} shows how any standard binary classifier can be \deleted{rendered}\added{made} scalable by simply including the scaling parameter $\rho$ in an additive way with the classifier predictor\deleted{, i.e.}\added{. That is,} given the function $\hat{f}:\mathcal{X}\longrightarrow\R$ and its corresponding classifier $\hat{\phi}(\x)$\deleted{$ = -\sign(\hat{f}(\x))$} then the function $f_\thB(\x,\rho) = \hat{f}(\x) + \rho$ provides the scalable classifier $\phi_\thB(\x,\rho)$\deleted{$ = -\sign(f(\x,\rho))$}. Thus, examples of classifiers that can be rendered scalable are support vector machine (SVM), support vector data description (SVDD), logistic regression (LR) but also artificial neural networks. More in detail, given a learning set 
\[
\mathcal{Z}_\ell\doteq\left\{\left(\x_i,y_i\right)\right\}_{i=1}^n \subseteq \mathcal{X}\times\left\{-1,+1\right\}
\]
containing observed feature points and corresponding labels\added{,} $\z_i=\left(\x_i,y_i\right)$\added{,} and assuming that $\varphi:\mathcal{X}\longrightarrow\mathcal{V}$ represents a \emph{feature map} (where $\mathcal{V}$ is an inner product space) that allows to exploit kernels, some examples of scalable classifier predictors are:
\bigskip
\begin{itemize}
    \item SVM: $f_\thB(\x,\rho) = \w^\top\varphi(\x) - b + \rho$,
    \bigskip
    \item SVDD: $f_{\thetaB}(\x,\rho) = \norm{\varphi(\x)-\w}^2 - R^2 + 
    \rho$,
    \bigskip
    \item LR: $f_{\thetaB}(\x,\rho) = \dfrac{1}{2}-\dfrac{1}{1+e^{\left(\w^\top\varphi(\x)-b\right) + \rho }}$,
\end{itemize}
\bigskip
where the classifier elements $\w,b$ and $R$ can be obtain\added{ed} as solution of proper optimization problems.  The interested reader can refer to \cite[Section~II~c]{carlevaro2023probabilistic}\deleted{)} for a more in depth discussion.
%
\\
Different values of the parameter $\rho$ correspond to different classifiers that can be considered as the level sets of the classifier predictor with respect to $\rho$. In particular, since we are interested in predicting the class $+1$ which, we recall, encodes a safety condition, we introduce
\begin{equation}
\mathcal{S}(\rho) = \set{\x\in\mathcal{X} }{f_\thB(\x,\rho)<0},
\label{eq:levelset}
\end{equation}
that is the set of points $\x\in\mathcal{X}$ predicted as safe by the classifier with the specific choice \added{of} $\rho$, i.e. the \emph{safety region} of the classifier $f_\thB$ for given $\rho$. It is easy to see that these sets are decreasingly nested with respect to $\rho$, i.e.
\begin{equation*}
\rho_1 > \rho_2 
\Longrightarrow \mathcal{S}(\rho_1) \subset \mathcal{S}(\rho_2). 
\end{equation*}
\subsection{Conformal Prediction}

Conformal Prediction is a relatively \deleted{new}\added{recent} framework developed \deleted{starting} in the late nineties \deleted{and early two thousand} by V.\ Vovk. We refer the reader to the surveys \cite{gentleIntro,JMLR:v9:shafer08a,fontana2020conformal} for a gentle introduction to this methodology. CP is mainly an a-posteriori verification of the designed classifier, and in practice returns a measure of its ``conformity'' to the calibration data.
We consider the particular implementation of CP discussed in \cite{gentleIntro}, relative to the so-called ``inductive" CP: in this setting, starting from a given predictor and a 
\textit{calibration} set, CP allows to construct a new predictor with given probabilistic guarantees.\\
To this end, the first key step is the definition of a 
\textit{score function} $s:\mathcal{X}\times\mathcal{Y} \longrightarrow \R$. Given a point $\x\in\mathcal{X}$ and a \textit{candidate} label $\y\in\{-1,1\}$, the score function returns a score $s(\x,\y).$ Larger scores encode worse agreement between point $\x$ and the candidate label $\y$. Then, assume to have available a second set of $n_c$ observations, usually referred to as \textit{calibration} set, defined as follows 
\begin{equation}
\label{eq:Zc}
    \mathcal{Z}_c \doteq \left\{(\x_i,y_i)\right\}_{i=1}^{n_c} = \mathcal{X}_c\times\mathcal{Y}_c\added{\subseteq \mathcal{X}\times\mathcal{Y}},
\end{equation} that are pairs of points $\x$ with their corresponding true label $y$.  \\
We assume that the observations $\x_i\in\mathcal{X}_c$ come from the same distribution $\Pr$ of the observations in the test set $\mathcal{Z}_{ts} = \{(\x_{i},y_i)\}_{i=1}^{n_{ts}} = \mathcal{X}_{ts}\times\mathcal{Y}_{ts}\added{\subseteq \mathcal{X}\times\mathcal{Y}}$. Additionally, CP requires that the data are \emph{exchangeable}, which is a weaker assumption than that of i.i.d.. Exchangeability means that the joint distribution of the data $\z_1, \z_2, \dots, \z_n$ is unchanged under permutations:
$$(\z_1,\z_2,\dots,\z_n) \sim (\z_{\sigma(1)},\z_{\sigma(2)}, \dots, \z_{\sigma(n)}), \ \textrm{for all permutations } \sigma. $$
Then, given a user-chosen confidence rate $(1-\varepsilon)\in(0,1)$, a \emph{conformal set} $C_\varepsilon(\x)$ is defined as the set of candidate labels whose score function is lower than the 
$(
\ceil{(n_c+1)(1-\varepsilon)}/n_c)
$-quantile, denoted as $s_\varepsilon$, computed on the $s_1 , \dots, s_{n_c}$ calibration scores. That is, to every point $\x$, CP associates a set of ``plausible labels"
$$C_\varepsilon(\x) = \set{\y\in\{-1,1\}}{s(\x,\y) \le s_\varepsilon}.$$
The usefulness of the conformal set is that, according with \cite{VovkGuarantee}, $C_\varepsilon(\x)$  possesses the so-called \emph{marginal conformal coverage guarantee} property, that is, given any (unseen before) observation $(\tilde \x,\tilde y)$, the following holds
\begin{equation}
\label{eq:marginal_cover}
\Pr\left\{\tilde y\in C_\varepsilon(\tilde{\x})\right\} \ge 1-\varepsilon.
\end{equation}
In other words, the true label $\tilde y$ belongs with high probability -- at least ($1-\varepsilon$) -- to the conformal set.

\section{Notion of Score Function for Scalable Classifiers and Conformal Safety Sets}
In this section we introduce two concepts: $i)$ a definition of score function for scalable classifiers (see Definition \ref{def:score_SCs}) and $ii)$ the notion of \emph{conformal safety region} (see Definition \ref{def:css}). 
\subsection{Natural definition of Score Function for Scalable Classifiers}
\label{sec:score}

In this paragraph, we show how scalable classifiers allow \added{for} a natural definition of the score function, based on their own classifier predictor.

\begin{definition}[Score Function for Scalable Classifier]
Given a scalable classifier $\phi_\theta(\x,\rho)$ with classifier predictor $f_\theta(\x,\rho)$, \added{given a point $\x$ and an associated} \deleted{and a} candidate label $\y$, the score function associated to the scalable classifier is defined as
$$s(\x,\y) = -\y\bar{\rho}(\x)$$
with $\bar{\rho}(\x)$ such that $f_\thB(\x,\bar\rho(\x))=0$.
\label{def:score_SCs}
\end{definition}

We notice that, since $f_\thB$ is a SC predictor, the existence and uniqueness of such $\bar\rho(\x)$ is guaranteed (Section \ref{sec:scalable}) and consequently $s$ is well defined.

In practice, the score function evaluates how much it is necessary to vary the original \deleted{classifier} \added{classification boundary} $f_\thB(\x,0)$ such that the point $\x$ falls on the classification boundary of the new classifier $f_\thB(\x,\bar\rho(\x))$, starting from class $\y$. Alternatively, it is possible to think of the score function as a measure of the ``difficulty'' of making the classifier predict a certain class: very large values for \deleted{$\rho(\x)$}$\added{\bar\rho(\x)}$ imply that it is difficult to render $f_\thB(\x,$\deleted{$\rho(\x)$}$\added{\bar\rho(\x)})$ positive, or equivalently that the class $-1$ is not conformal (thus, when $\hat{y}=-1$, the score function is \deleted{$\rho(\x)$}$\added{\bar\rho(\x)}=-\hat{y}$\deleted{$\rho(\x)$}$\added{\bar\rho(\x)}$). Very negative values of \deleted{$\rho(\x)$}$\added{\bar\rho(\x)}$ imply that it is difficult to render the output equal to $+1$, thus the score function is in this case $-$\deleted{$\rho(\x)$}$\added{\bar\rho(\x)} = -\hat{y}$\deleted{$\rho(\x)$}$\added{\bar\rho(\x)}$\deleted{)}.
  
\bigskip

\begin{example}
Scalable SVDD is the most straightforward example of correctly understanding such a definition for score function. In this case the score function takes this form
$$s(\x,\y) = -\y\left(R^2-\norm{\w-\varphi(\x)}^2\right).$$
\deleted{It is}\added{This represents} exactly the quantity that needs to be removed ($\y=+1$ for the point inside the sphere, $\norm{\w-\varphi(\x)}^2-R^2 < 0$) or added ($\y=-1$ for the point outside the sphere, $R^2-\norm{\w-\varphi(\x)}^2 > 0$) to the radius such that $\x$ falls on the boundary of the classifier.\\
For example, consider two classes of points, ``safe'' ($+1$, in blue in the following figure) and ``unsafe'' ($-1$, in red in the following figure) sampled from two two-dimensional Gaussian distributions with respectively means and covariance matrices
$$\mu_\mathrm{S} = \begin{bmatrix}
-1 \\
-1
\end{bmatrix}, \, \sigma_\mathrm{S} = \frac{1}{2}\mathrm{I} \, \, ; \, \,
\mu_\mathrm{U} = \begin{bmatrix}
+1 \\
+1
\end{bmatrix}, \, \sigma_\mathrm{U} = \frac{1}{2}\mathrm{I} $$
where $\mathrm{I}$ is the identity matrix. We trained a linear SVDD classifier (Figure \ref{fig:ex1a}) and plotted the respectively score function (Figure \ref{fig:ex1b}). Exactly the behavior described above can be observed: the score function associates values according to the geometry provided by the classifier. In this case, points belonging to the boundary of the circumference have score function values of 0 (dashed green line) and negative or positive depending on whether the point is inside or outside the circumference. It is worth noting that the classifier can be interpreted as a level set of the score function, and this interpretation is crucial as will become clear in the following.
\begin{figure}[!h]
  \subfigure[Linear SVDD classifier]{%
    \includegraphics[width=0.5\textwidth]{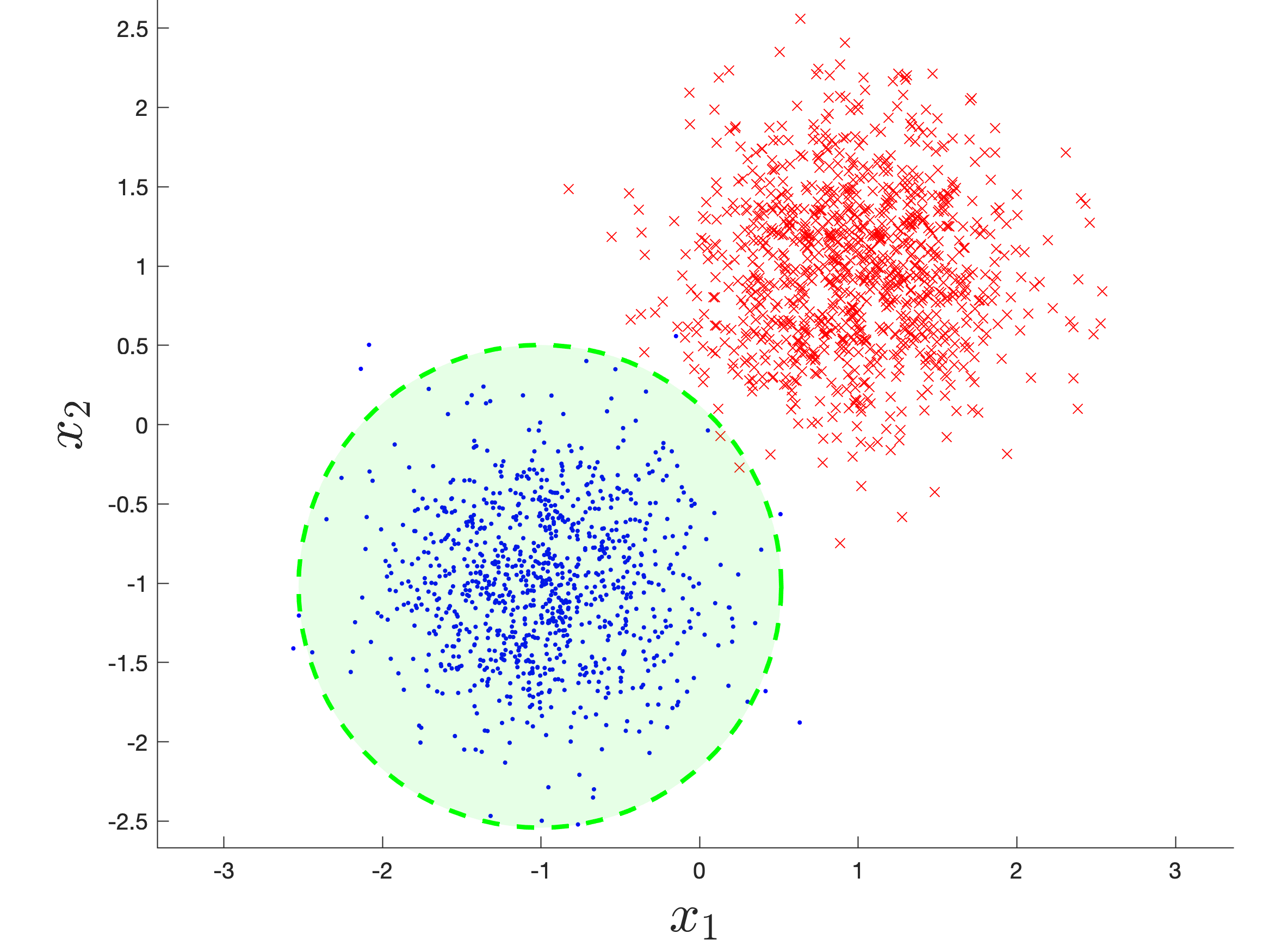} 
    \label{fig:ex1a}
  }
  \subfigure[Linear SVDD score function]{%
    \includegraphics[width=0.5\textwidth]{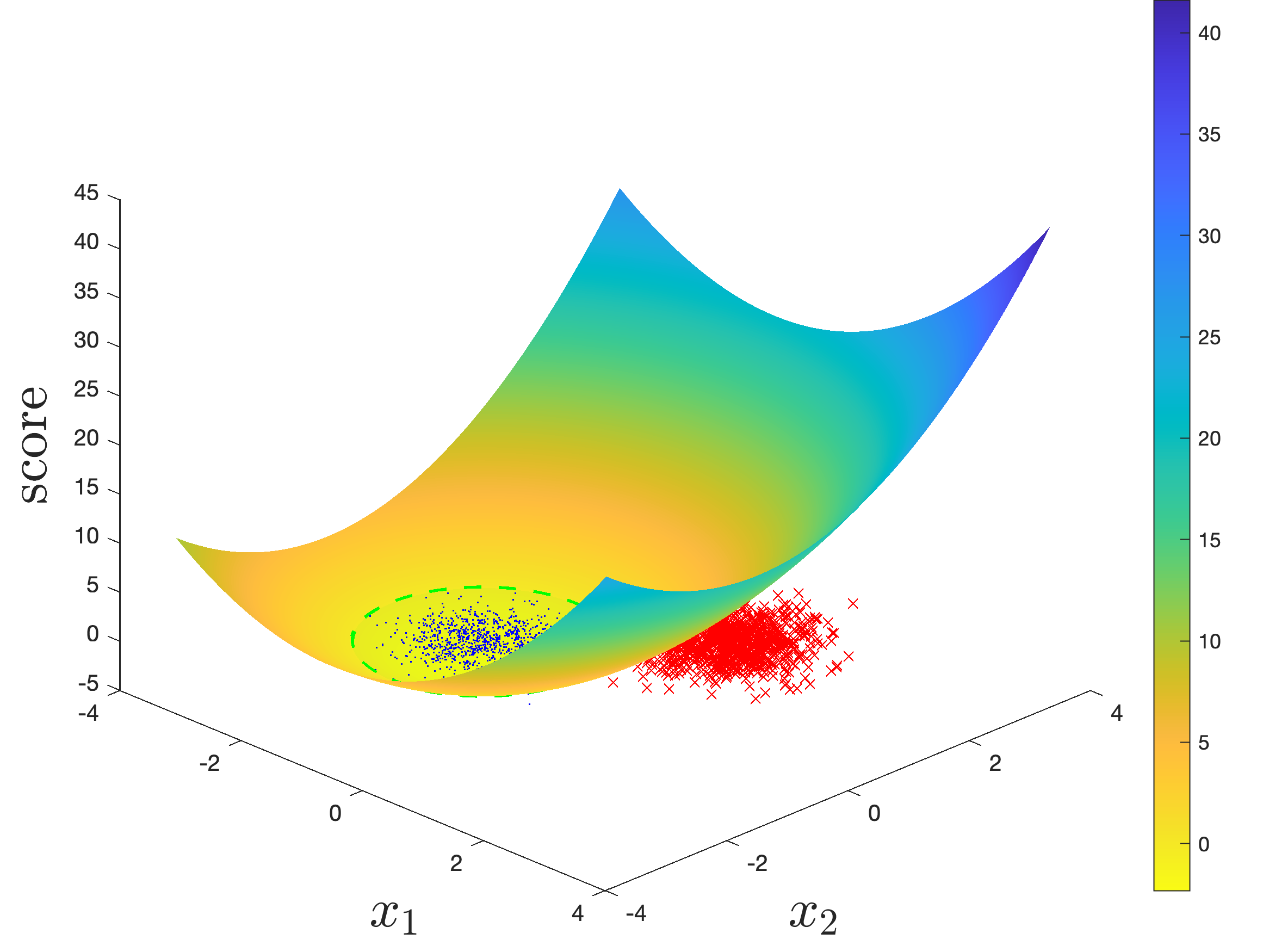} 
    \label{fig:ex1b}
  }
  \caption{Relationship between the SVDD classifier and the corresponding score function: the absolute value of the score function assigns to a sample its distance to the circumference boundary. \added{The color bar on the right helps to understand the behavior of the score function: darker colors indicate regions with less conformity with the target class, warmer the opposite. The zero value of the score function is obtained exactly on the boundary. }}
  \label{fig:example1}
\end{figure}
\label{ex:1}
\end{example}
\subsection{Conformal Safety Regions}
\label{sec:csr}

Classical CPs define subsets of the output space that satisfy the probabilistic marginal coverage constraint, but it is equally important to understand the relationship between the input space and the conformal sets. In other words, it would be meaningful to define regions in the input space classified on the basis of the conformal set of their samples to identify for which inputs the classifier is most reliable in making a certain prediction.
For example, one should be interested in finding the region of classification uncertainty ($C_\varepsilon(\x)=\{-1,+1\}$) or the region in which the conformal classifier predicts a specific label ($C_\varepsilon(\x)=\{+1\}$ or $C_\varepsilon(\x)=\{-1\}$) or in which it has no guess at all ($C_\varepsilon(\x) = \varnothing$). 
\\
In particular, since the goal is to find the input values that bring the classification to a ``safe'' situation (i.e., in our notation, $y=+1$) with a certain level of confidence, we introduce the concept of \emph{conformal safety region}.
\begin{definition}[Conformal Safety Region]
\label{def:css}
Consider a calibration set $\mathcal{Z}_{c} = \{(\x_i,y_i)\}_{i=1}^{n_c}$ from the same data distribution of the test set $Z_{ts}$. Given a level of error $\varepsilon\in (0,1)$, a score function $s:\mathcal{X}\times\mathcal{Y}\longrightarrow\R$\added{,} and \deleted{the}\added{its} corresponding $(\lceil(n_c+1)(1-\varepsilon)\rceil/n_c)$-quantile $s_\varepsilon$ computed on the calibration set, the {\rm{conformal safety region (CSR) of level $\varepsilon$} }  is defined as follows
\begin{equation}
    \begin{split}
    \Sigma_\varepsilon = \set{\x\in\mathcal{X}}{ s(\x,+1)\le s_\varepsilon, \ s(\x,-1)>s_\varepsilon}.
    \end{split}
\end{equation}
\end{definition}

In words, a conformal safety region (CSR) is the subset of the input space where the conformal set is composed by only safe labels, $C_\varepsilon(\x)=\{+1\}$, which can be inferred directly from the definition. Note that the above definition is independent on the choice of the score function $s$. What we will prove in the next is that using the score function defined for SCs (Definition \ref{def:score_SCs}) it is possible to give an analytical form to $\Sigma_\varepsilon$.
%
%
\begin{example}
    Consider the same configuration as in Example \ref{ex:1} but with covariance matrices $\sigma_S = \sigma_U = I$ and with a probability to sample an outlier for each class $p_O = 0.1$. Consider the LR classifier and its corresponding score function
    $$s(\x,\y) = -\y(b-\w^\top\varphi(\x))\textrm{\deleted{.}}\added{,}$$
    \added{which is the same of the SVM since the solution of the equation $f_\thB(\x,\rho)=0$ is for both $\bar\rho(\x) = b-\w^\top\varphi(\x)$.}
    We trained on a training set composed by $3000$ samples the LR classifier with cubic polynomial kernel (Figure \ref{fig:ex2a}) and then we computed the score values on a calibration set of $5000$ samples. We computed the quantiles varying $\varepsilon$ ($0.05, 0.1$ and $0.5$) and we plotted (on a test set of $10000$ samples) the scatter of the points according to the conformal set. Green points belong to the CSR $\Sigma_\varepsilon$ and it is easily understandable that the smaller \deleted{is} $\varepsilon$ \added{is}, the smaller \deleted{is} $\Sigma_\varepsilon$. This behavior is in line with CP theory, since small values of $\varepsilon$ mean that the conformal prediction must be very precise, and this is achievable only if the classifier itself is ``very confident'' of assigning the true label to a sample. Also, it should be noted that the smaller $\varepsilon$ is, the larger the region of uncertainty for the conformal prediction ($C_\varepsilon(\x) =\{-1,+1\}$, in yellow in Figures \ref{fig:ex2a}, \ref{fig:ex2b}). Again, since for small $\varepsilon$ high levels of marginal coverage must be satisfied, conformal prediction tends to give both labels to a point when it is uncertain.  Contrarily, for high values of $\varepsilon$ (Figure \ref{fig:ex2c}) the conformal sets for uncertain points tend to be empty (in \deleted{purple}\added{black}) because the score is too high and no output meets the specifications to belong to $C_\epsilon$. Finally, it is worth noting that the regions into which the points scatter have a well-defined shape: as introduced in Example \ref{ex:1} and as will become clear in the next section, these regions correspond to level sets of the score function.   
\label{ex:2}
    \begin{figure}[!t]
\begin{center}
  \subfigure[LR classifier]{%
    \includegraphics[width=0.28\textwidth]{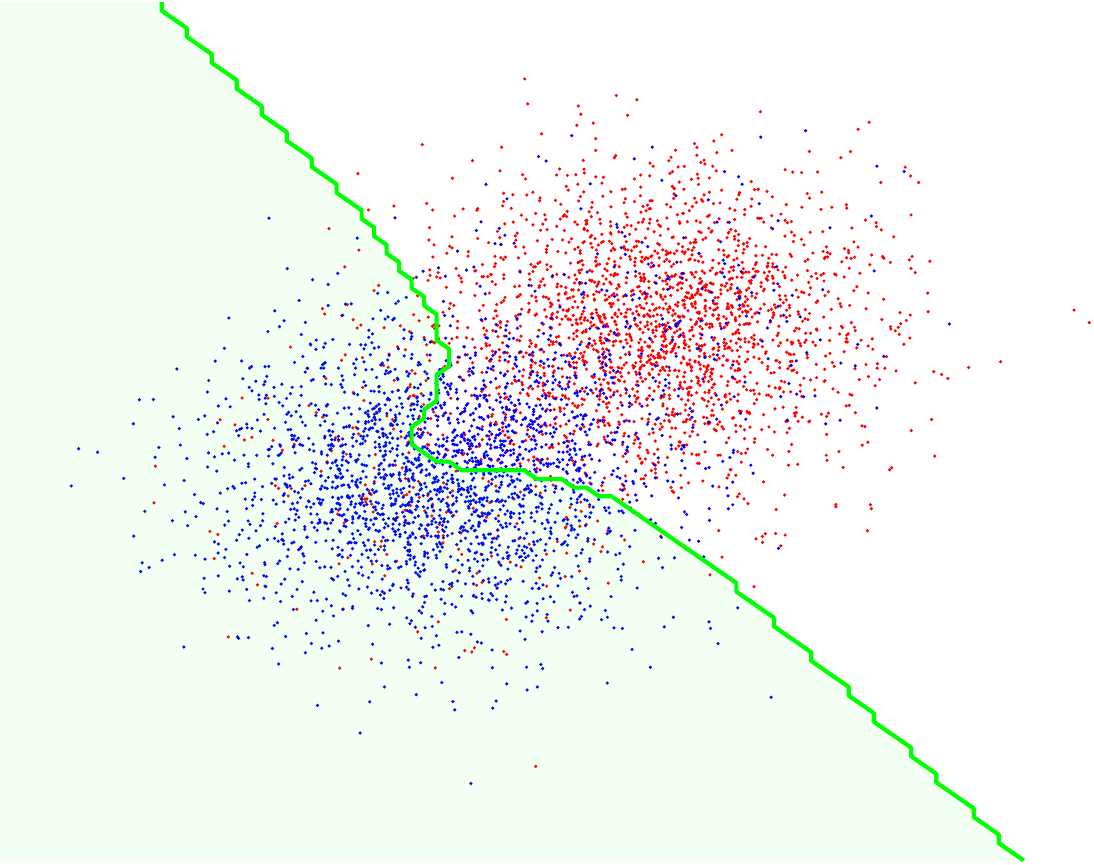} 
    \label{fig:ex2a}
  }
  \vfill
  \subfigure[$\varepsilon = 0.05$]{%
    \includegraphics[width=0.28\textwidth]{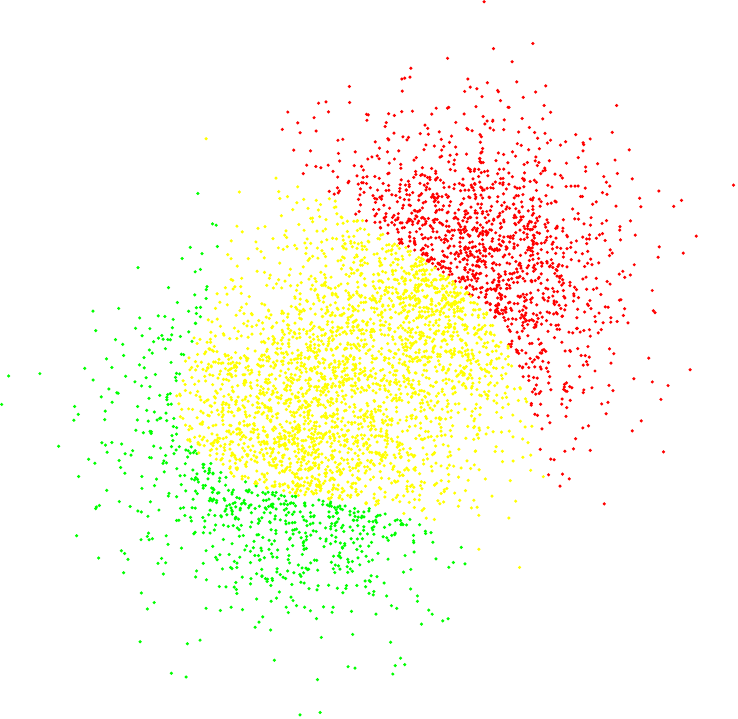} 
    \label{fig:ex2b}
  }
  \subfigure[$\varepsilon = 0.1$]{%
    \includegraphics[width=0.28\textwidth]{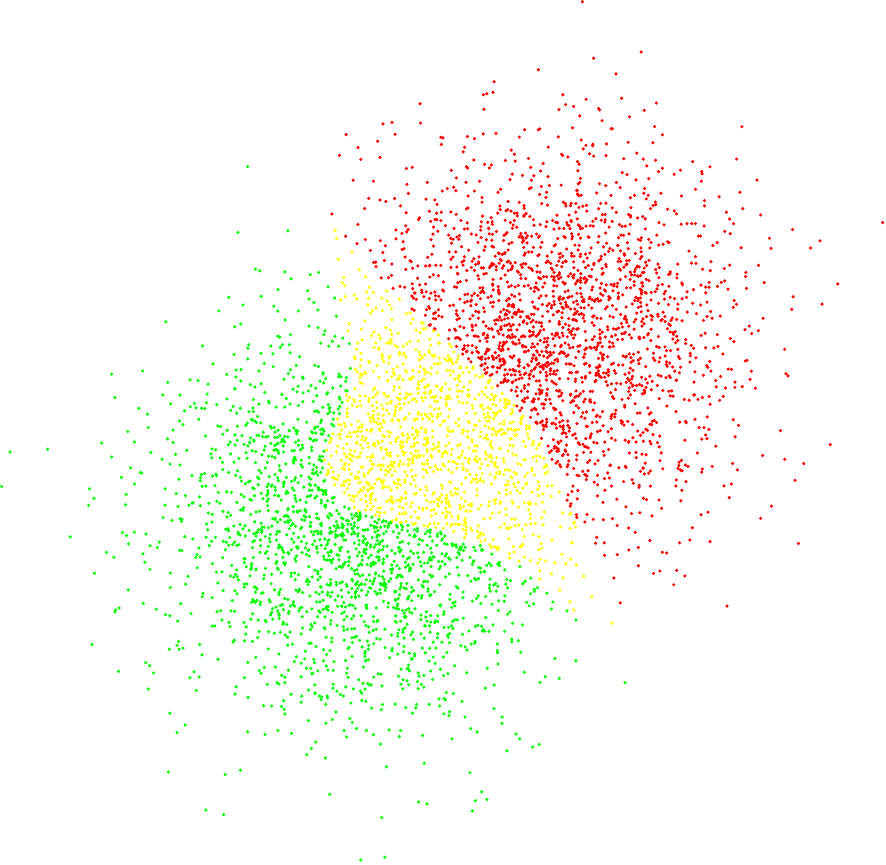} 
    \label{fig:ex2c}
  }
  \subfigure[$\varepsilon = 0.5$]{%
    \includegraphics[width=0.28\textwidth]{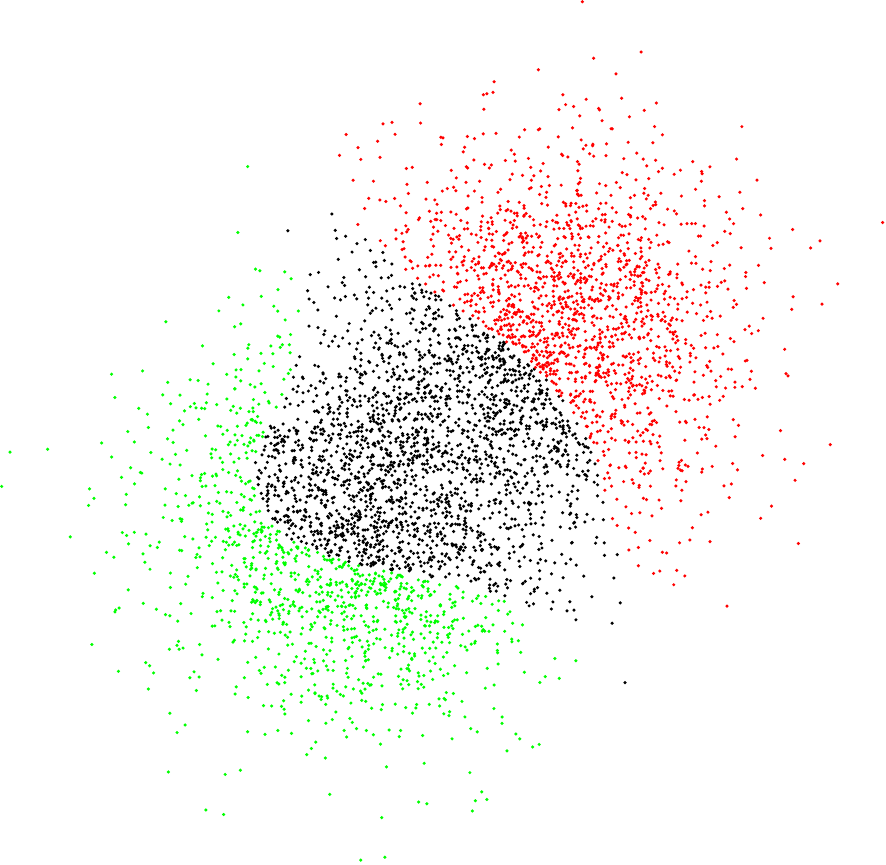} 
    \label{fig:ex2d}
  }
  \caption{Scatter-plots of the conformal set varying $\varepsilon$ for cubic LR. Green and red points correspond to singleton conformal set ($C_\varepsilon(\x)=\{+1\}$ and $C_\varepsilon(\x)=\{-1\}$ respectively) yellow points to double predictions ($C_\varepsilon(\x)=\{+1,-1\}$) and \deleted{purple}\added{black} points to empty prediction ($C_\varepsilon(\x)=\varnothing$).}
  \label{fig:example2}
\end{center}
\end{figure}
\end{example}

\subsection{Analytical Form of Conformal Safety Regions for Scalable Classifiers}

The definition of score we gave for SCs in Definition \ref{def:score_SCs} identifies a particular value of the scalable parameter, which is the one corresponding to the quantile $s_\varepsilon$, that we can define formally as
\begin{equation}
    \label{eq:opt-radius}   \rho_\varepsilon\doteq|s_\varepsilon|.
\end{equation}
To this value, we can associate a level set $\mathcal{S}(\rho_\varepsilon)$
defined as in \eqref{eq:levelset}, i.e. the $\rho_\varepsilon$-safe set
\begin{equation}
    \mathcal{S}_\varepsilon =  \set{\x\in\mathcal{X} }{f_\thB(\x,\rho_\varepsilon)<0}.
\end{equation} 
We can prove that non-trivial relationships link $\mathcal{S}_\varepsilon$ to the CSR $\Sigma_\varepsilon$. But before, let us split $\Sigma_\varepsilon$ in two contribution:
\begin{equation}
    \Sigma_\varepsilon = \Sigma_\varepsilon^a \cup \Sigma_\varepsilon^b, 
\end{equation}
where
\begin{equation}
\Sigma_\varepsilon^a=\set{\x\in\mathcal{X}}{s(\x,+1)<s_\varepsilon, s(\x,-1)>s_\varepsilon},
\end{equation}
\added{and}
\begin{equation}
\Sigma_\varepsilon^b=\set{\x\in\mathcal{X}}{s(\x,+1)=s_\varepsilon, s(\x,-1)>s_\varepsilon}.
\end{equation}
The relationship between $\mathcal{S}_\varepsilon$ and $\Sigma_\varepsilon$ is \deleted{made through}\added{explored in} the following results that provide as final and major contribution the fact that $\mathcal{S}_\varepsilon\subseteq\Sigma_\varepsilon$.
\begin{proposition}
    $$\mathcal{S}_\varepsilon = \Sigma_\varepsilon^a \subseteq \Sigma_\varepsilon.$$
    \label{prop:Sigmaa}
\end{proposition}
\begin{proof}
    \begin{eqnarray*}
        \x\in\mathcal{S}_\varepsilon &\iff& f_{\boldsymbol{\theta}}(\x,|s_\varepsilon|)<0,\\
        &\iff& f_{\boldsymbol{\theta}}(\x,|s_\varepsilon|)<f_{\boldsymbol{\theta}}(\x,\bar\rho(\x)),\\
        &\iff& |s_\varepsilon|<\bar\rho(\x),\\
        &\iff&-s_\varepsilon<\bar\rho(\x) \ \textrm{and} \ s_\varepsilon < \bar\rho(\x),\\
        &\iff& -s_\varepsilon <-s(\x,+1) \ \textrm{and} \ s_\varepsilon<s(\x,-1),\\
        &\iff& s(\x,+1)<s_\varepsilon \ \textrm{and} \ s(\x,-1)>s_\varepsilon,\\
        &\iff& \x \in\Sigma_\varepsilon^a\subseteq\Sigma_\varepsilon.
    \end{eqnarray*}
\end{proof}
\begin{corollary}
    $$\mathcal{S}_\varepsilon = \Sigma_\varepsilon \ \ \textrm{only if} \ \ \Sigma_\varepsilon^b = \varnothing.$$
\label{cor:Sigmabvoid}
\end{corollary}
\begin{proof}
    Trivial, from 
    $$\Sigma_\varepsilon = \Sigma_\varepsilon^a\cup\Sigma_\varepsilon^b = \mathcal{S}_\varepsilon\cup\Sigma_\varepsilon^b.$$
\end{proof}
\begin{proposition}
    $$\Sigma_\varepsilon^b\neq\varnothing \Longrightarrow s_\varepsilon>0.$$
\label{prop:Sigmabnotvoid}
\end{proposition}
\begin{proof}
    \begin{eqnarray}
        \x\in\Sigma_\varepsilon^b &\iff& s(\x,+1) = s_\varepsilon \ \textrm{and} \ s(\x,-1)<s_\varepsilon,\\
        &\iff&-\bar\rho(\x)=s_\varepsilon \ \textrm{and} \ \bar\rho(\x)<s_\varepsilon,\\  
        &\iff& -s_\varepsilon<s_\varepsilon,\\
        &\iff& s_\varepsilon>0.
        \end{eqnarray}
\end{proof}
We can then summarize all these information in a single theorem that defines the ``analytical form'' of the CSR, i.e. that it is possible to express $\Sigma_\varepsilon$ in terms of a single scalar parameter.
\begin{theorem}[Analytical Representation of the Conformal Safety Region via Scalable Classifiers]
\label{teo-main}
Consider the classifier \eqref{eq:phi:minus} and suppose that \cite[Assumption 1]{carlevaro2023probabilistic} holds and that $\Pr\{\x\in\mathcal{X}\} = 1$. Consider then a calibration set $\mathcal{Z}_c = \left\{(\x_i,y_i)\right\}_{i=1}^{n_c}$ ($n_c$ exchangeable samples), a level of error $\varepsilon\in(0,1)$, a score function $s:\mathcal{X}\times\mathcal{Y}\longrightarrow\R$ as in Definition \ref{def:score_SCs} with $\lceil(n_c+1)(1-\varepsilon)\rceil/n_c$-quantile $s_\varepsilon$ computed on the calibration set\added{.} \deleted{and}\added{D}\deleted{d}efine the {\rm{conformal scaling of level}} $\varepsilon$ as follows
\begin{equation}
    \rho_\varepsilon = |s_\varepsilon|,
\end{equation}
and define the corresponding $\rho_\varepsilon$-safe set
\begin{equation}
    \mathcal{S}_\varepsilon =  \set{\x\in\mathcal{X} }{f_\thB(\x,\rho_\varepsilon)<0}.
\end{equation} 
Then, given the conformal safety region of level $\varepsilon$, $\Sigma_\varepsilon$, we have
\begin{enumerate}[i)]
    \item $\mathcal{S}_\varepsilon \subseteq \Sigma_\varepsilon$.
    \item $\mathcal{S}_\varepsilon = \Sigma_\varepsilon$ if $s_\varepsilon\le0$.
\end{enumerate}
that is, $\mathcal{S}_\varepsilon$ is a CSR.
\end{theorem}
\begin{proof}
    Proof follows directly from Propositions \eqref{prop:Sigmaa} and \eqref{prop:Sigmabnotvoid} and Corollary \eqref{cor:Sigmabvoid}.
\end{proof}
    \begin{figure*}[!t]
\begin{center}
\subfigure[CP sets scatter-plot.]{%
    \includegraphics[width=0.45\textwidth]{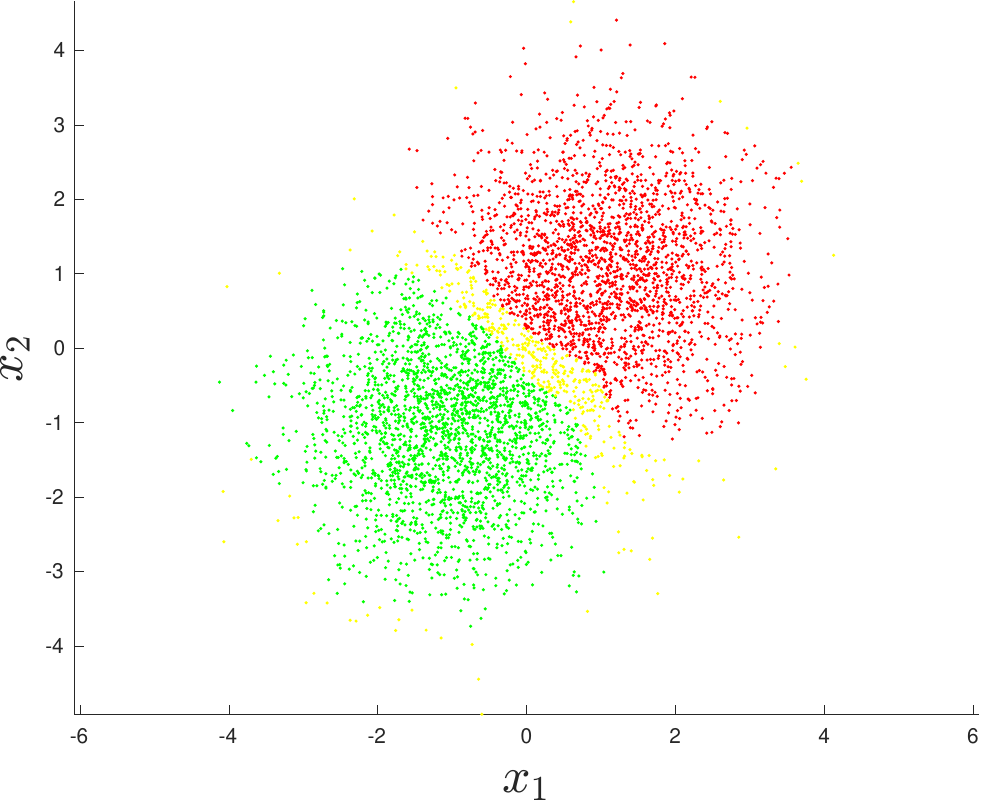} 
    \label{fig:ex3a}
  }
  \hfill
  \subfigure[CSR $\mathcal{S_\varepsilon
  }$.]{%
    \includegraphics[width=0.45\textwidth]{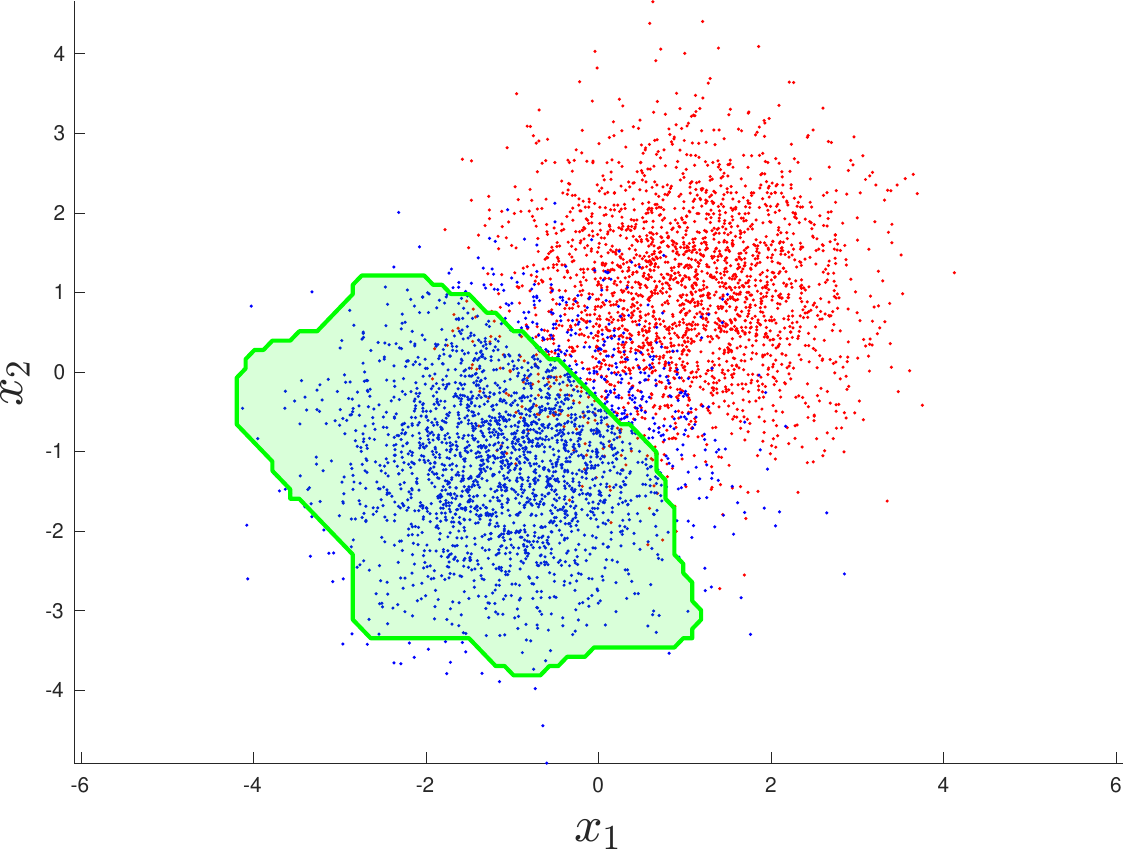} 
    \label{fig:ex3b}
  }
  \vfill
  \subfigure[Gaussian SVM score function.]{%
    \includegraphics[width=0.45\textwidth]{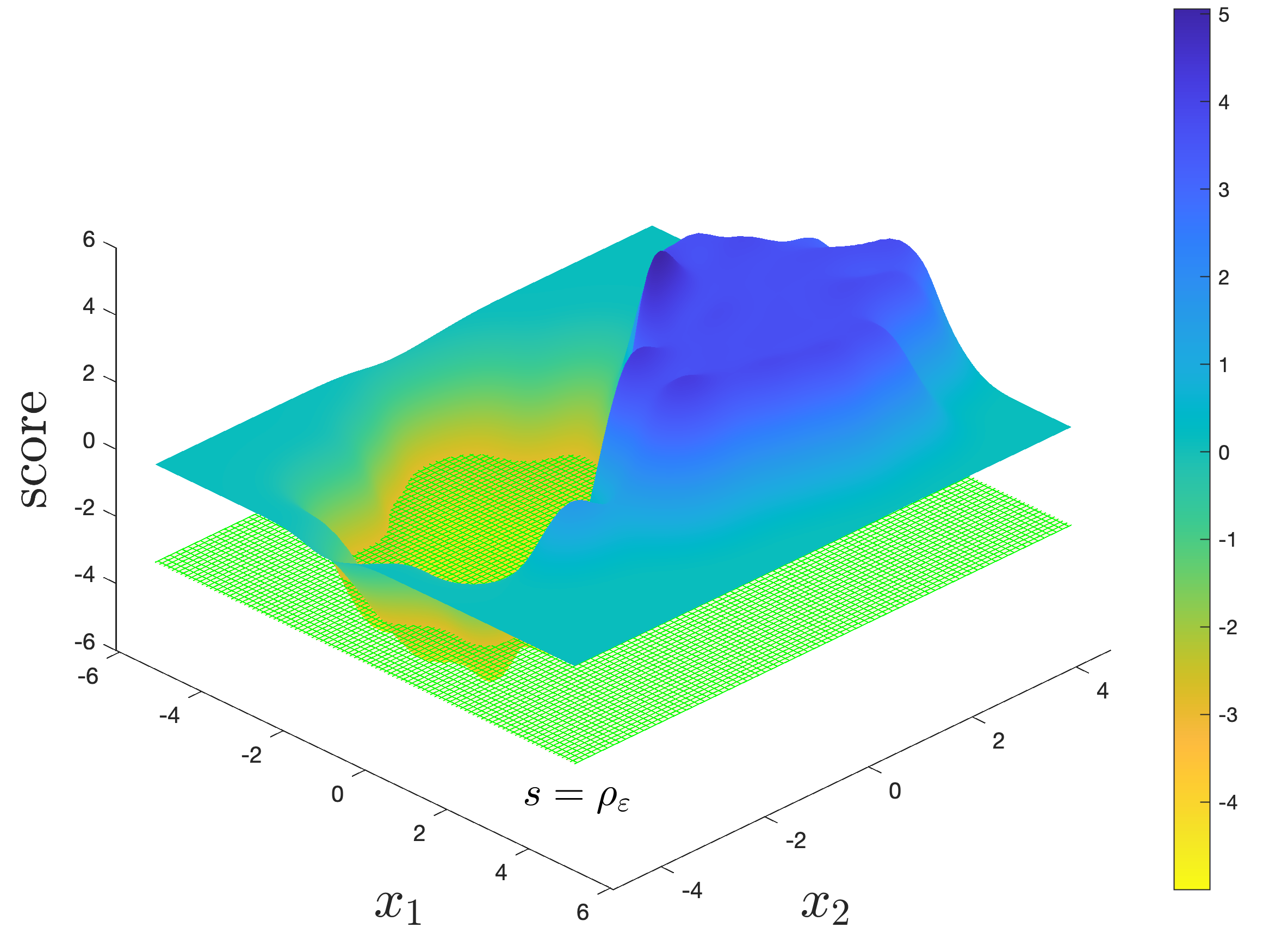} 
    \label{fig:ex3c}
  }
  \hfill
  \subfigure[$\textrm{``score''}=0$ plane projection.]{%
    \includegraphics[width=0.45\textwidth]{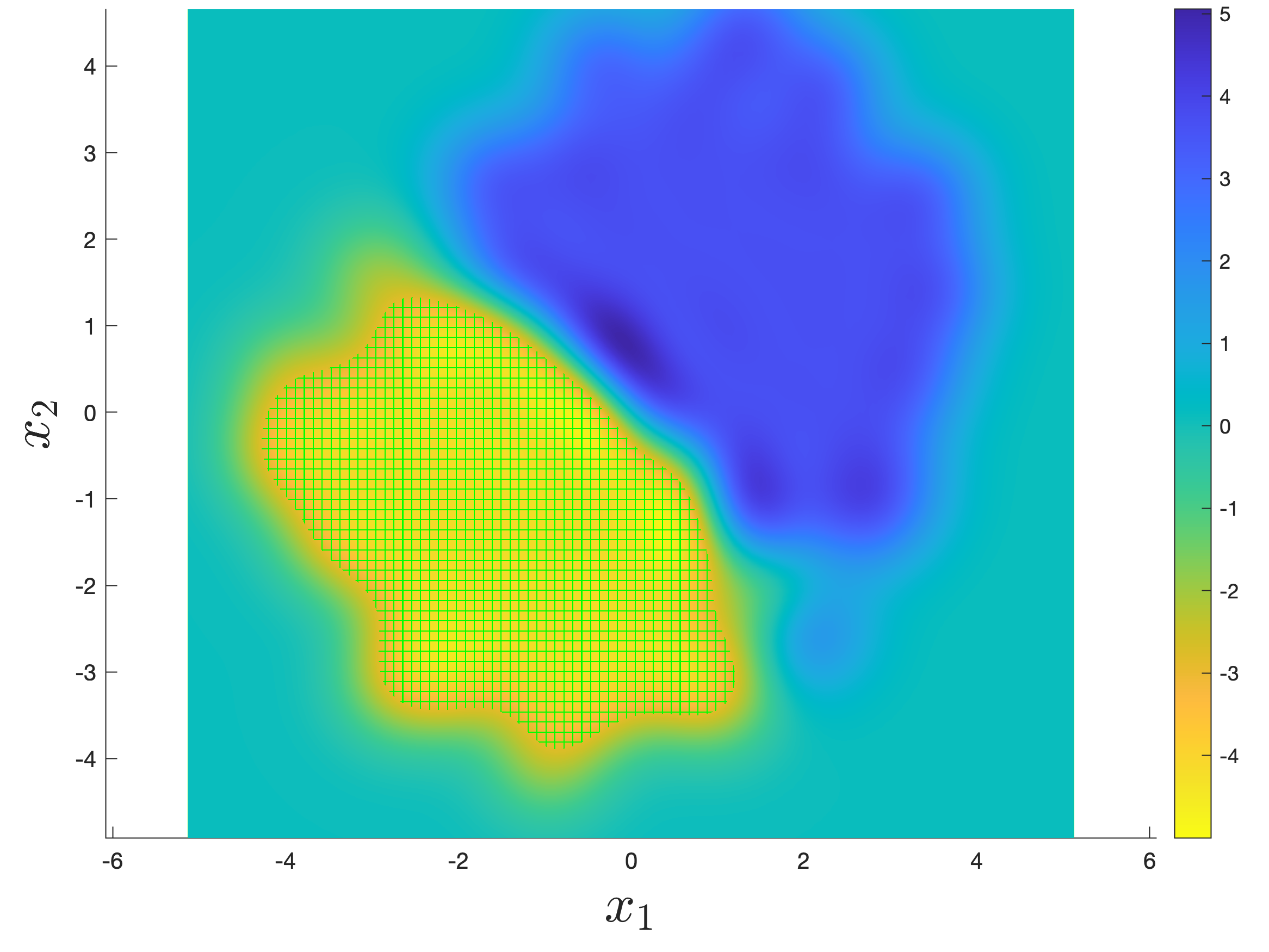} 
    \label{fig:ex3d}
  }
  \caption{CSR computed with a Gaussian SVM at $\varepsilon = 0.05$. Scattered CSR $\Sigma_\varepsilon$, \ref{fig:ex3a}, coincides with the analytical CSR $\mathcal{S}_\varepsilon$, \ref{fig:ex3b} that coincides with the level set $z = \rho_\varepsilon$ of the score function, \ref{fig:ex3c}. \added{Figure \ref{fig:ex3d} is the planar representation on ${x_1}-{x_2}$ plane of the score function.}}
  \label{fig:example3}
\end{center}
\end{figure*}
In its classical definition, conformal prediction is a local
property, \deleted{that is conformal coverage guarantee is valid
only punctually}\added{that is, the conformal coverage guarantee is valid only punctually}. However, conformal labels map each
point in a subset of the input space, depending on the size
of the respective conformal set. \deleted{In this sense Theorem \ref{teo-main} proves that it is possible
to identify a region in which the possibility of observing a
safe situation (in our encoding $y = +1$) is guaranteed.}\added{Theorem 3 then provides a new classifier that maps the samples contained in $\mathcal{S}_\varepsilon$ to the target class $+1$. Once $\rho_\varepsilon$ has been computed, it is then possible to write }
    \begin{equation*}
       \added{\mathcal{S}_\varepsilon = \phi_\thB(\cdot,\rho_\varepsilon)^{-1}(y=+1)},
    \end{equation*}
    \added{identifying a unique relationship between the target class of the classification and the CSR.}

\begin{example}
In the same configuration as in Example \ref{ex:2}, we trained a Gaussian SVM and calculated the score values on the calibration set. Figure \ref{fig:example3} shows exactly what Theorem \ref{teo-main} claims: CSRs are level sets of the score function that correspond to a specific quantile and thus to a specific confidence level. Specifically, in this example it is shown the CSR at level of confidence $1-0.05 = 0.95$ that results in a quantile equal to $2.8113$ and corresponding conformal scaling $\rho_{0.05} = -2.8113$. The hyperplane $\textrm{``score''}= -2.8113$ exactly cuts the score function at the level set corresponding to the CSR.
\end{example}

\begin{remark}[On the Usefulness of Conformal Safety Regions]
The introduction of the concept of CSR brings inevitably to understand how this instrument can be useful in practice. First of all it allows to identify reliable prediction regions and quantify uncertainty: in decision making problems where a certain amount of confidence in the prediction is required (like for example in medical applications) CSRs can suggest the best set of input features that guides the predictions reliably, minimizing the presence of misclassification samples. Moreover CSRs provide an interpretable way to understand the model's behavior in different regions of the input space. This can be useful for the model explanation and for possible improvements and corrections to the model. Finally, CSRs are very ``regulatory compliant'': in applications with regulatory requirements, CSRs ensure compliance by providing a clear understanding of where model's predictions are reliable.
\end{remark}
In addition, CSRs can provide strong information about the prediction of points belonging to them. Indeed, it can be proved that the number of false positives is limited by the $\varepsilon$ error.
\begin{theorem}
Consider the classifier \eqref{eq:phi:minus} and the corresponding CSR developed as in Theorem \ref{teo-main} with a level of error $\varepsilon\in(0,1)$. Then, it can be stated that
    \begin{equation}
    \Pr\left\{y=-1 \ \textrm{and} \ \x \in \mathcal{S}_\varepsilon\right\} \le \varepsilon.
    \label{idea_bella}
\end{equation}
\label{conjecture}
\end{theorem}

\begin{proof}
    Since $\mathcal{S}_\varepsilon\subseteq\Sigma_\varepsilon $:
    \begin{equation*}
    \begin{split}
        \Pr\{y = -1 \ \textrm{and} \ \x\in\mathcal{S}_\varepsilon \} &\le \Pr\{y = -1 \ \textrm{and} \ \x\in\Sigma_\varepsilon \}\\
        &= \Pr\{y = -1 \ \textrm{and} \ \x\in\{\x : C_\varepsilon(\x)=\{+1\}\} \}\\
        &\le \Pr\{y = -1 \ \textrm{and} \ \x\in\{\x : -1 \notin C_\varepsilon(\x)\} \}\\
        &= \Pr\{y = -1 \ \textrm{and} \ y\notin C_\varepsilon(\x) \}\\
        & \le \Pr\{y \notin C_\varepsilon(\x) \} \le \varepsilon,
    \end{split}
    \end{equation*}
    where the last inequality holds for the marginal coverage property of CP \eqref{eq:marginal_cover}.
\end{proof}

%
%
The significance of this statement cannot be overstated, as it implies that thanks to CSRs, it becomes feasible to identify regions in feature space where the conformal coverage of the target class is assured. Consequently, these regions identify feature points with a high degree of certainty, thereby enhancing the reliability, trustworthiness, and robustness of (any) classification algorithm, especially with regard to safety considerations.
\added{Specifically, the final output of the proposed method is a region, $\mathcal{S}_\varepsilon$, in which with high probability the chance of finding the unwanted label is small (and thus as small as desired). This means that the scalable classifier together with the conforming prediction can handle the natural uncertainty arising both from the data (to the extent that the data are representative of the information they provide, i.e. \emph{aleatoric uncertainty}) and the model (to the extent that it is accurate in modeling, i.e. \emph{epistemic uncertainty}), providing ``safety'' sets that have a volume proportional to $\varepsilon$, i.e., to the confidence of the prediction \cite{hullermeier2021aleatoric}.This is very much in line with recent and ongoing literature in the field of geometric uncertainty quantification, as in \cite{sale2023volume} where the authors propose the idea of ``credal sets'' \cite{abellan2006disaggregated} that, as our CSR does, guarantee the correctness of the prediction bounding the input set in polytopes. In this regard, the idea of quantifying uncertainty through functions that give a measure of distance (such as the score function proposed here) is something that is sparking the UQ community, enabling future comparisons with other methods such as the ``second order UQ'' discussed in \citep{sale2023second}.  }
\begin{remark}[On the link with Probably Approximate Correct theory]
          \added{Probably approximate correct (PAC) learning is a theory developed in the 1980s by Leslie Valiant \cite{valiant2013probably} for quantifying uncertainty in learning processes, with a focus on the case of undersampled data. PAC learning has been used to define sets of predictions that can satisfy probabilistic guarantees with nonparametric probabilistic assumptions (see, for example, \cite{park2022pac}) with similarities with our (and in general with CP theory) approach. Specifically, PAC learning is a broad theory where it is possible to insert the research presented in this paper on uncertainty quantification of machine learning classifiers with conformal prediction. For example, the confidence bounds on which conformal prediction theory is based (and so is our research) are inherited from PAC learning theory.}  \added{As shown in }\cite[Prop 2a]{inductive}\added{, the concept of $(\varepsilon,\delta)$-validity (i.e. the marginal coverage guarantee of equation (5) together with the randomness of the calibration set) is a PAC style guarantee on the (inductive) conformal prediction. As reported in our previous work \cite{carlevaro2023probabilistic}, there are nontrivial relationships between the number of samples on the calibration set and probabilistic guarantees on prediction. All these relationships can be read into the PAC learning formalism, and future work will focus on this topic.} 
\end{remark}

In the next section\deleted{, illustrated in Figure \ref{fig:example6},} we report some numerical examples of Theorem \eqref{conjecture}\added{, see Figure \ref{fig:example6}}.

\section{A real world application: detection of SSH-DNS tunnelling}
\label{sec:DNS}
The dataset chosen for the example application deals with covert channel detection in cybersecurity \cite{aiello2015dns}. The aim is detecting the presence of secure shell domain name server (SSH-DNS) intruders by an aggregation-based monitoring that avoids packet inspection, in the presence of silent intruders and quick statistical fingerprints generation. By modulating the quantity of anomalous packets in the server, we are able to modulate the difficulty of the inherent supervised learning solution via canonical classification schemes (\cite{9594676, Ivan}). \\
Let $q$ and $a$ be the packet sizes of a \textit{query} and the corresponding \textit{answer}, respectively (what answer is related to a specific query can be understood from the packet identifier) and $Dt$ the time-interval intercurring between them.
The information vector is composed of the statistics (mean, variance, skewness and kurtosis) of $q, a $ and $Dt$ for a total number of 12 input features:
$$\x=[m_A, m_{Q}, m_{Dt}, v_A, v_{Q}, v_{Dt}, s_A, s_{Q}, s_{Dt}, k_A, k_{Q}, k_{Dt}],$$
and an overall size of $10000$ examples.
High-order statistics give a quantitative indication of the asymmetry (skewness) and heaviness of tails (kurtosis) of a probability distribution, helping to improve the detection inference. The output space $\mathcal{Y} = \{-1,+1\}$ is generated by associating each sample $\x$ with the label $-1$ when ``no tunnel'' is detected and $+1$ when ``tunnel'' is detected. In this sense, the idea of safety should be interpreted as an indication that the system has detected the presence of a ``tunnel'' or abnormal behavior, i.e., the system believes that there is a potential security threat or intrusion. This could trigger various security countermeasures, such as blocking incoming traffic or applying filters to the connection. 
\\
   \begin{figure*}[!t]
\begin{center}
\subfigure[Error plot SVM.]{%
    \includegraphics[width=0.31\textwidth]{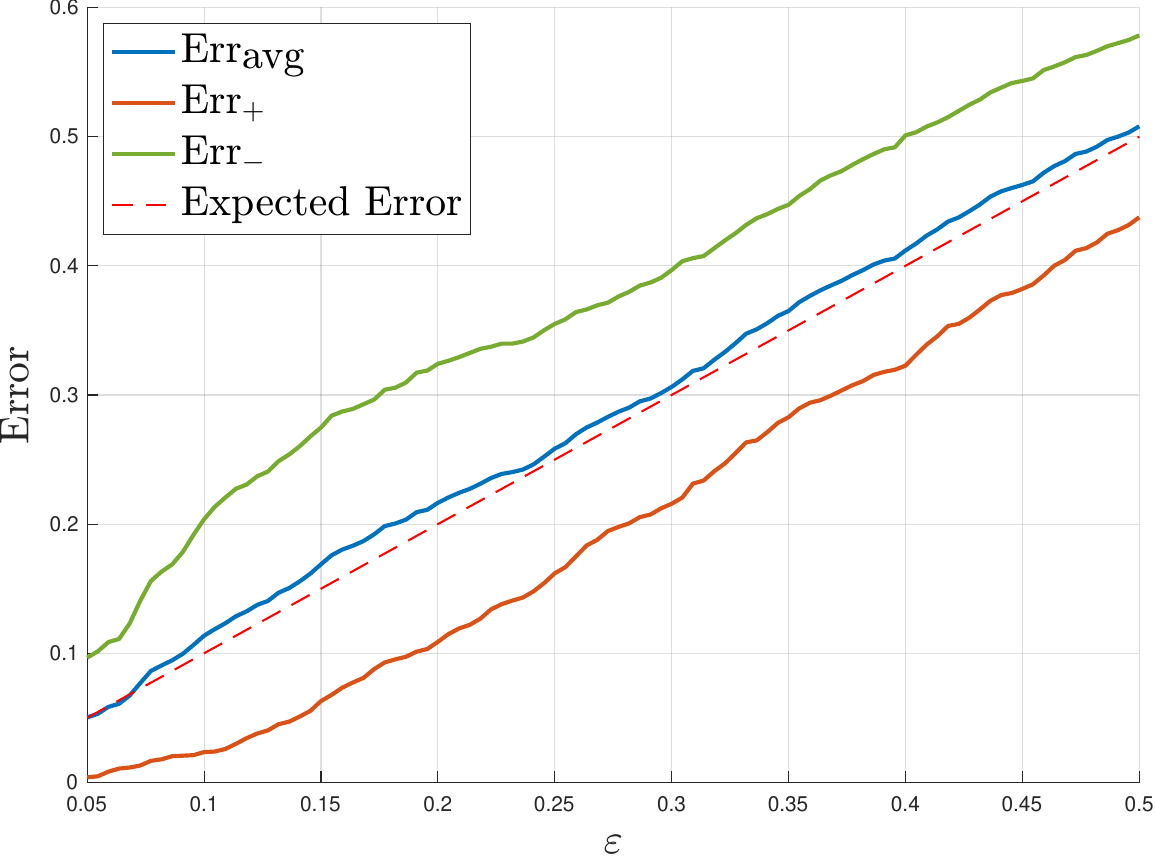} 
    \label{fig:ex4a}
  }
  \hfill
  \subfigure[Error plot SVDD.]{%
    \includegraphics[width=0.31\textwidth]{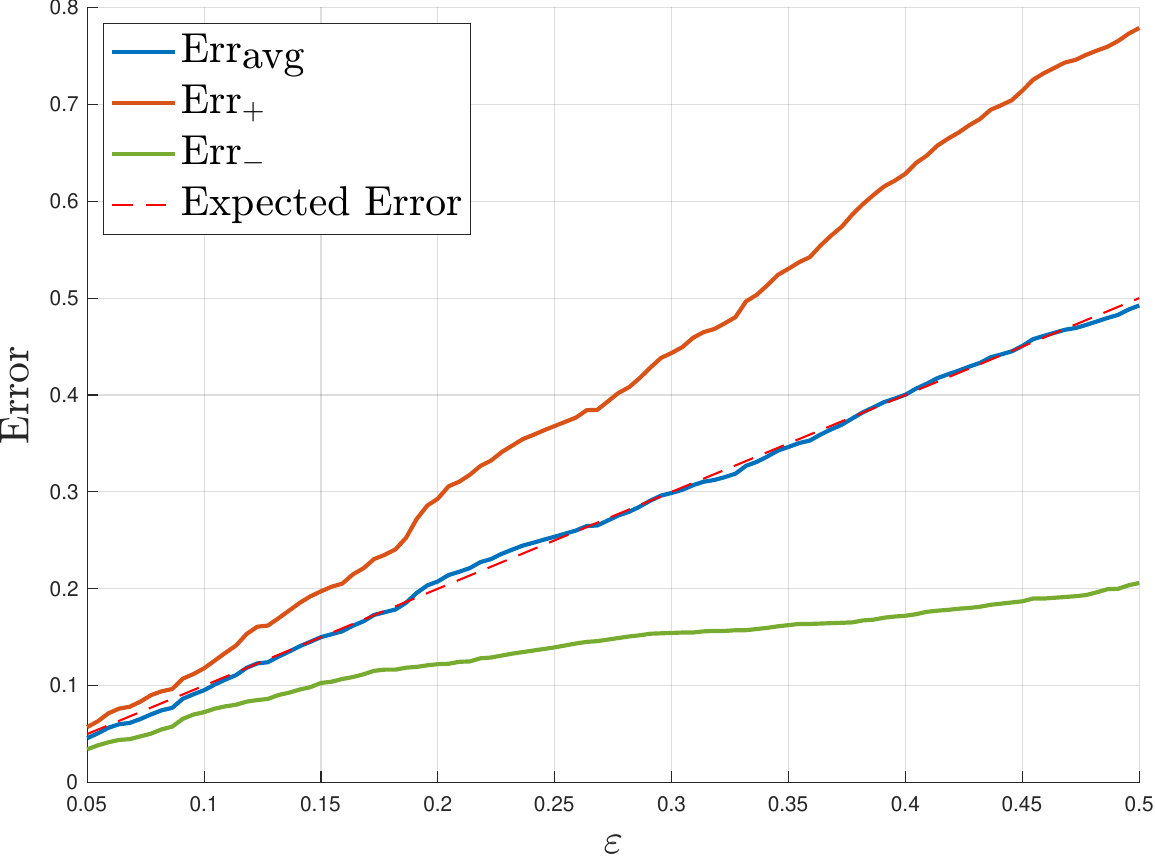} 
    \label{fig:ex4b}
  }
  \hfill
  \subfigure[Error plot LR.]{%
    \includegraphics[width=0.31\textwidth]{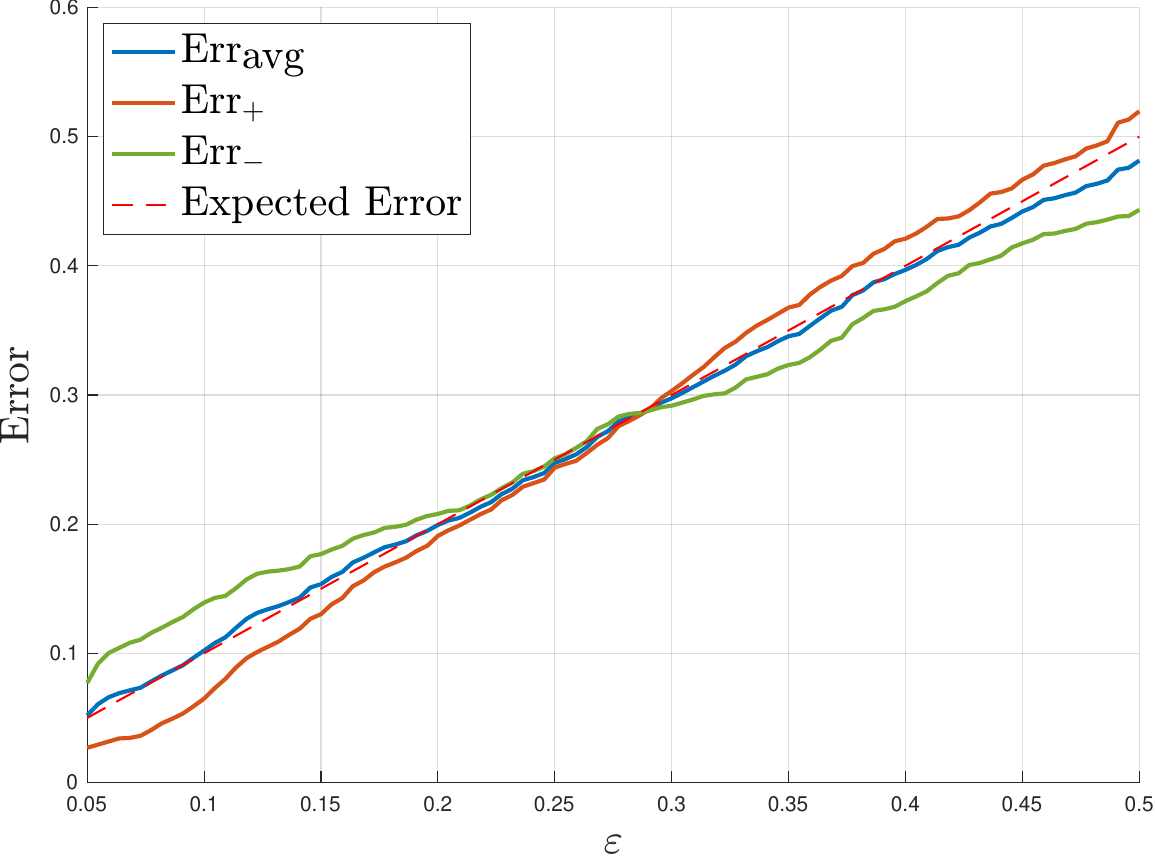} 
    \label{fig:ex4c}
  }
  \caption{Trend of the average error as $\varepsilon$ varies in $[0.05, 0.5]$ for different classifiers. \added{The errors vary in $[0,0.6]$ for SVM, $[0,0.8]$ for SVDD and $[0,0.6]$ for LR.} }
  \label{fig:example4}
\end{center}
\end{figure*}  

\begin{figure*}[!b]
\begin{center}
\subfigure[Size plot SVM.]{%
    \includegraphics[width=0.32\textwidth]{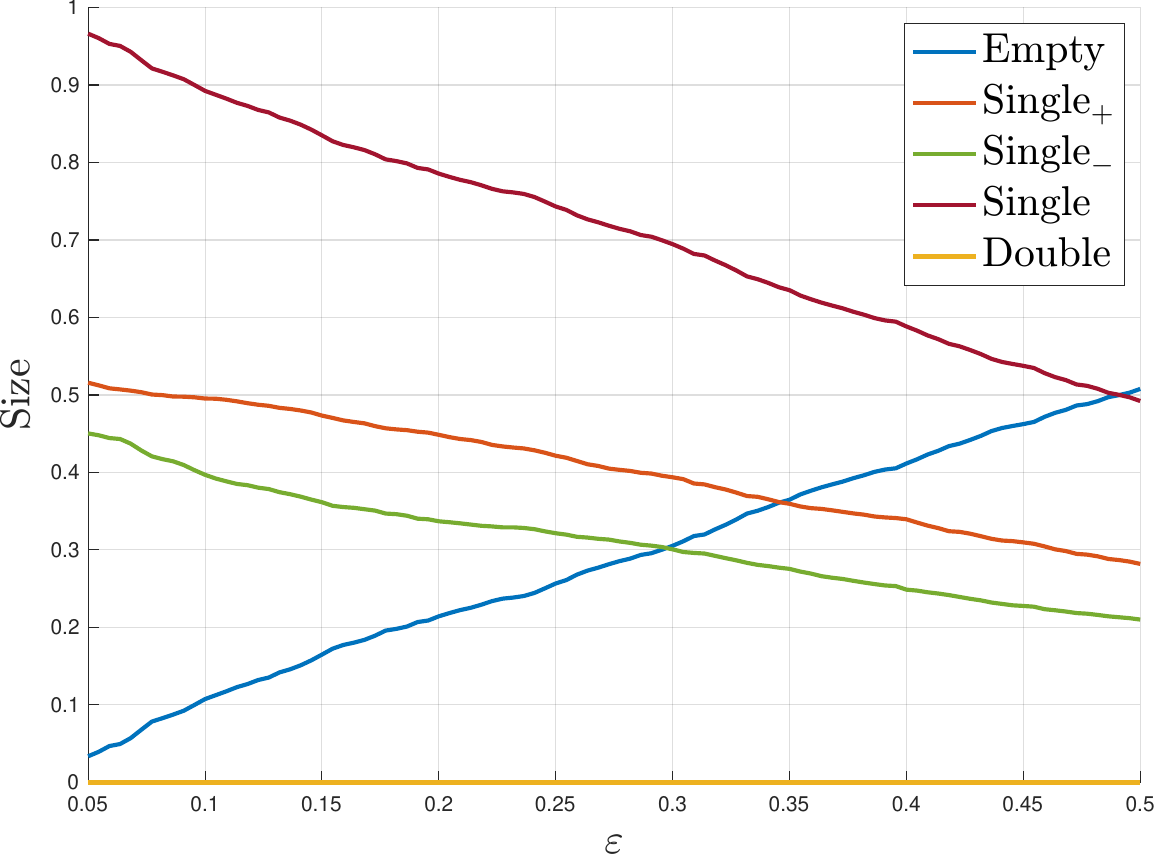} 
    \label{fig:ex5a}
  }
  \hfill
  \subfigure[Size plot SVDD.]{%
    \includegraphics[width=0.32\textwidth]{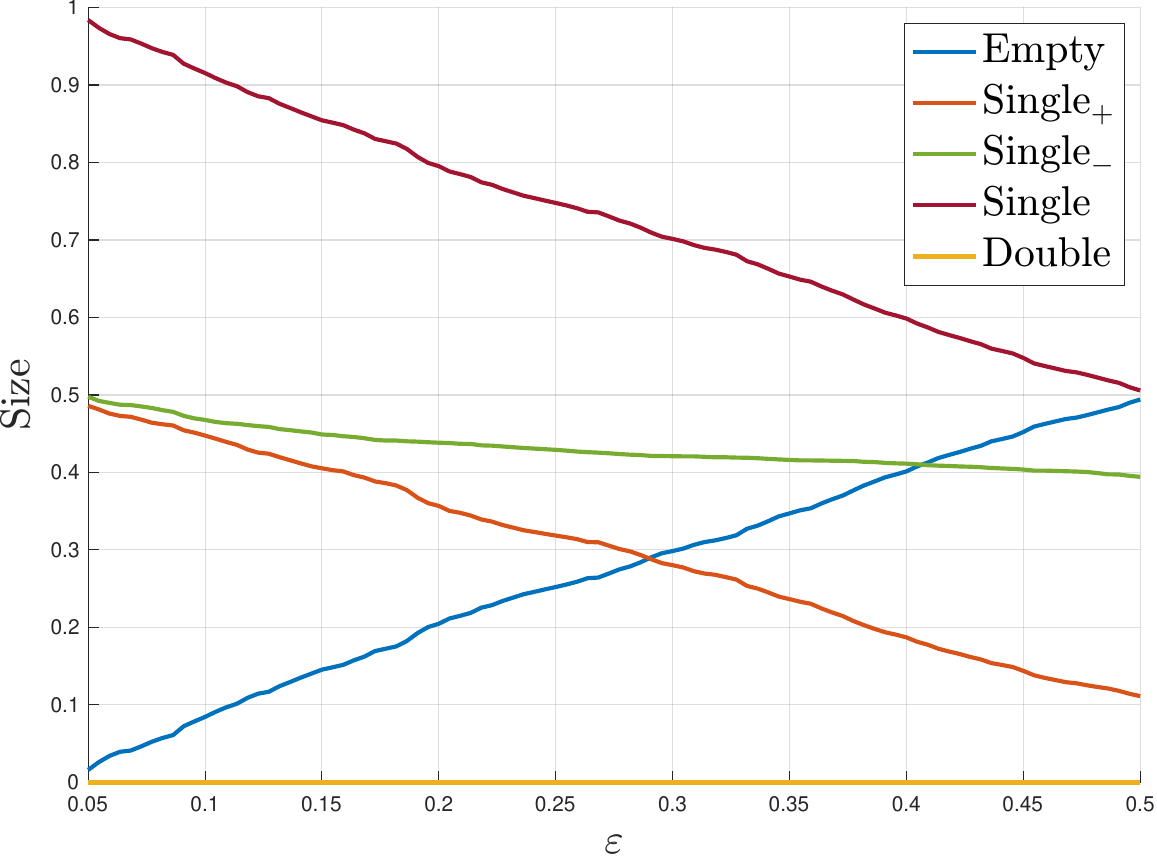} 
    \label{fig:ex5b}
  }
  \hfill
  \subfigure[Size plot LR.]{%
    \includegraphics[width=0.3\textwidth]{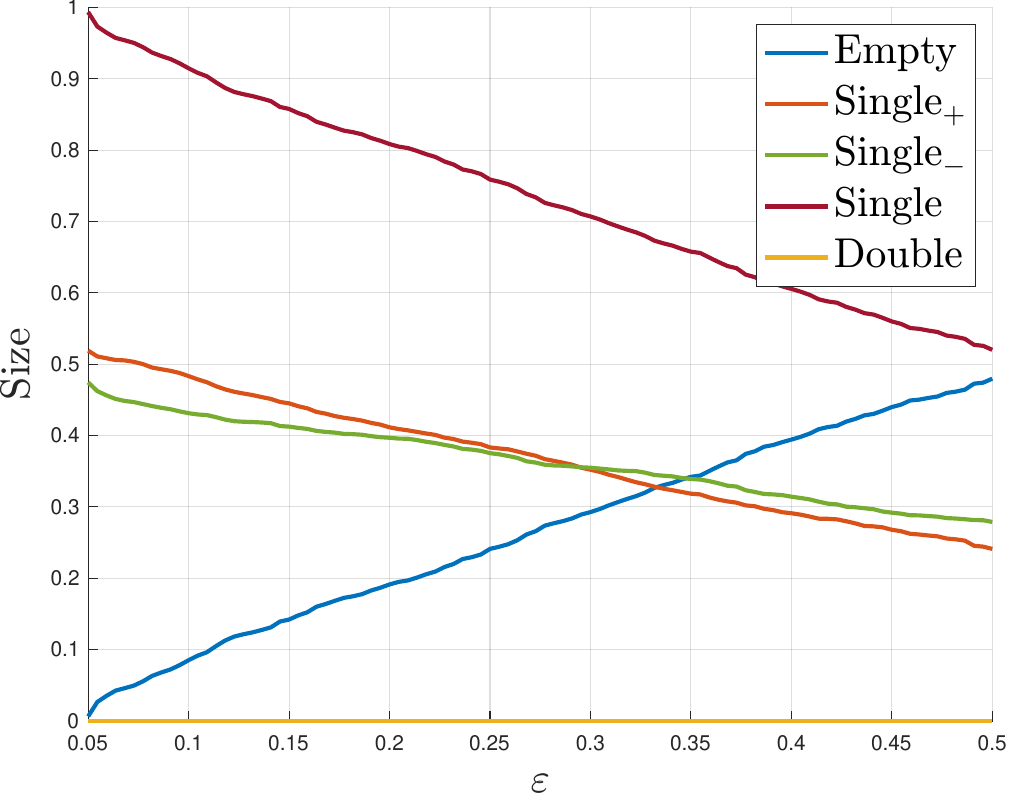} 
    \label{fig:ex5c}
  }
  \caption{Trend of the average size of conformal sets as $\varepsilon$ varies in $[0.05, 0.5]$ for different classifiers. \added{The size varies from 0 (empty) to 1 (full).}}
  \label{fig:example5}
\end{center}
\end{figure*}  

Conformal predictions assess the goodness of an algorithm by two basic metrics of evaluation: accuracy and efficiency. Accuracy is measured by the average error, over the test set, of the conformal prediction sets considering points of both classes (\emph{err}), only class $y=-1$ points ($\textrm{\emph{err}}_-$) and only class $y=+1$ points ($\textrm{\emph{err}}_+$). We remind that an error occurs whenever the true label is not contained in the prediction set. Efficiency is  quantified through the rate of test points prediction sets with no predictions (\emph{empty}), two predictions (\emph{double}) and  singleton predictions (\emph{single}), these ones also divided by class ($\emph{single}_-$ and $\emph{single}_+$).
The obtained results (as the classifier varies) are reported in Figures \ref{fig:example4} and \ref{fig:example5} for\deleted{, respectively,} accuracy and efficiency\added{, respectively}.\\
The overall metrics computed on the benchmark dataset outline the expected behavior of the conformal prediction, with slight differences between the example classifiers.
For all values of $\varepsilon$, the average error is indeed bounded by $\varepsilon$ in all cases. Also, \emph{err} increases linearly with $\varepsilon$. \emph{This means that the classification is reliable}.
As for the size of the conformal set, \deleted{results in their overall}\added{the overall results} point out that for small values of $\varepsilon$ the model produces more double-sized regions, since in this way it would be ``almost certain'' that the true label is contained in the conformal set. Then, the size reduces by increasing $\varepsilon$, allowing \added{for} the presence of more empty prediction sets. The number of singleton conformal set remains always sufficiently high (it increases as double conformal sets decrease and it decreases as  empty conformal set increase) meaning that \emph{the classification is efficient}.
Regarding the use of the example classifiers, it is interesting to note that LR is the most stable with respect to $\varepsilon$ and the error conditional on classes: the error rate for both classes is nearly linear with $\varepsilon$, suggesting that the prediction is reliable even conditional on the single class or, better, that the classifier is able to clearly separate the classes while maintaining the expected confidence. The same behavior is also observed for SVM, although the errors per class deviate more from the average error. The error for class ``tunnel'' is always lower than that for class ``no tunnel'', suggesting that the classifier is more likely to minimize the number of false positives, losing in accuracy for true negatives. The opposite behavior is observed for SVDD, which instead tries to classify negative instances better, resulting in a lower expected classification error for class ``no tunnel''. The most interesting aspect, however, is that the algorithm is less conformal when conditioned on the error of the single class, increasing the spread with respect to the average error as $\varepsilon$ increases. \added{Conformal prediction together with scalable classifiers define then a totally new framework to deal with uncertainty quantification in classification based scenarios. The results shown in this application drastically overcome the ones obtained on the same dataset in \cite{9594676}. The previous approach relied on an iterative procedure to control the number of misclassified points that could only be used with a specific algorithm (SVDD) and without a-priori confidence bounds, but only on the basis of an smart trial-and-error algorithm. The point that the reader should observe is precisely this: the presented theory allows dealing with the uncertainty naturally brought by machine learning approaches in a simple and probabilistically grounded way, allowing setting confidence in prediction by design.  }
Finally, Figure \ref{fig:example6} shows the behavior of the coverage erro\added{r} of the CSR with respect to the example classifiers. As stated in Theorem \ref{conjecture}\added{,} the probability that the wrong label $-1$ is predicted for the points belonging to $\mathcal{S}_\varepsilon$ is under-linear with respect the expected error $\varepsilon$. 

\begin{figure*}[!t]
\begin{center}
\subfigure[Error coverage SVM]{%
    \includegraphics[width=0.33\textwidth]{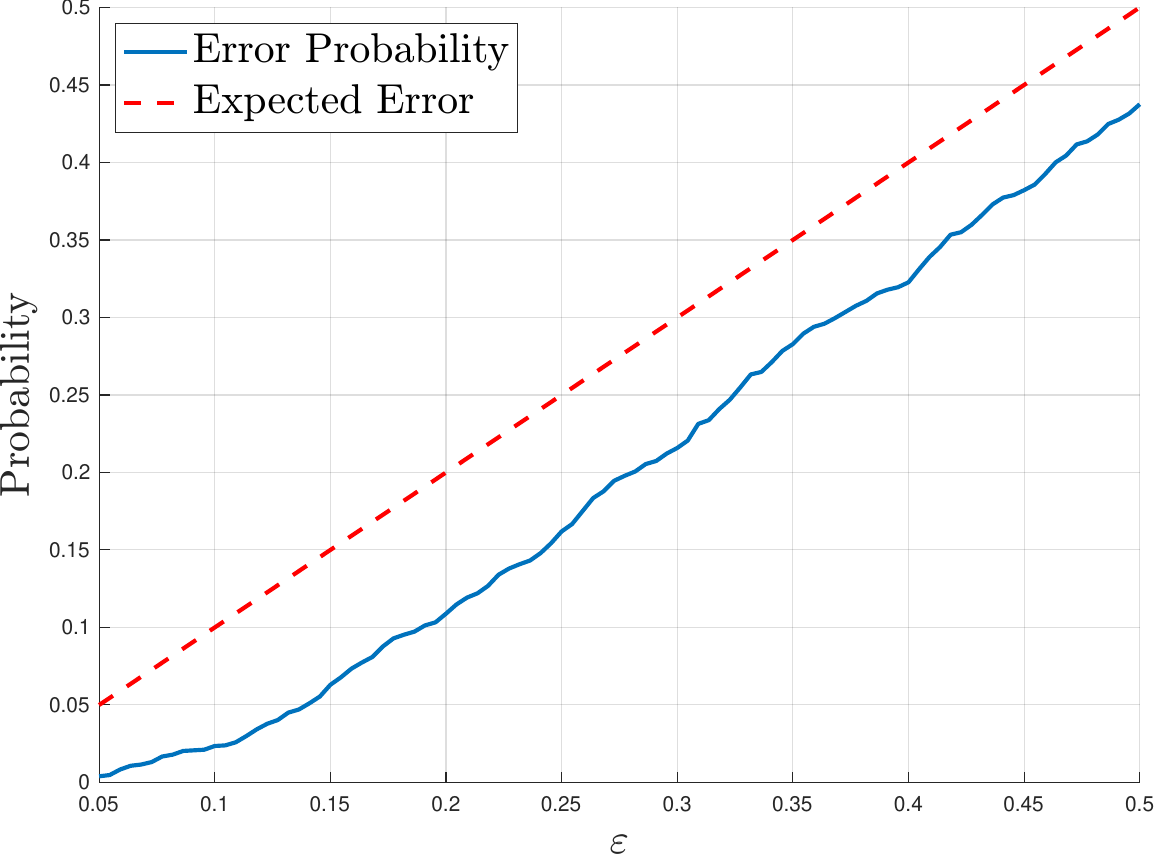} 
    \label{fig:ex6a}
  }
  \hfill
  \subfigure[Error coverage SVDD]{%
    \includegraphics[width=0.31\textwidth]{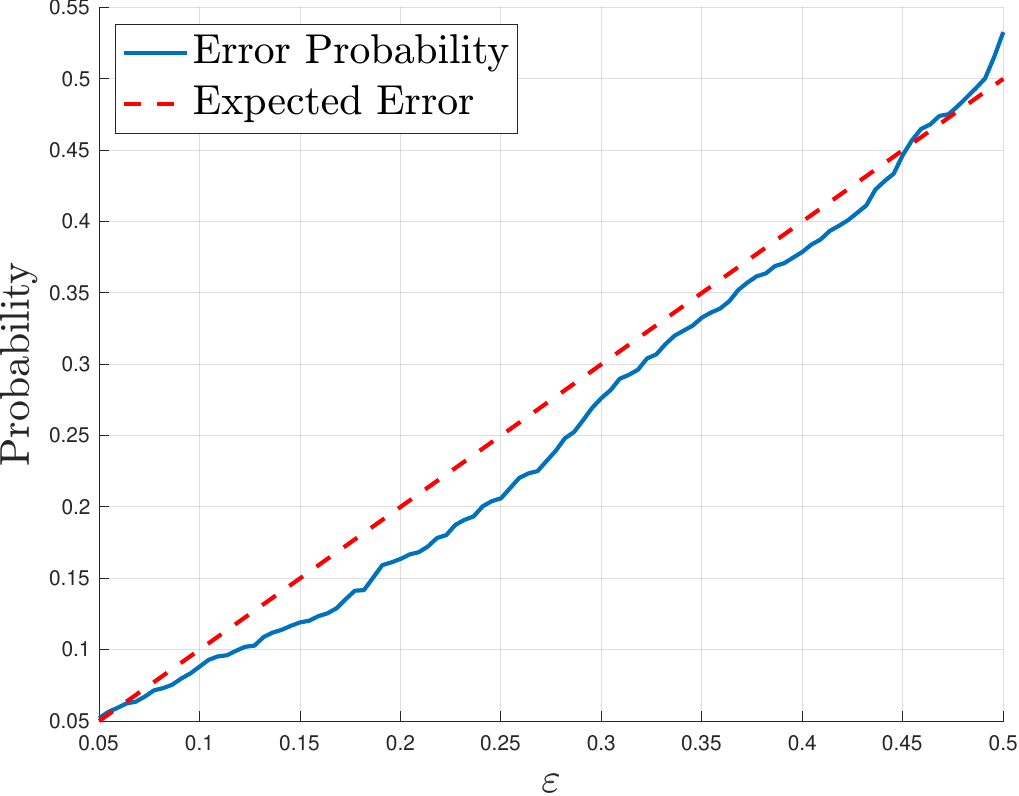} 
    \label{fig:ex6b}
  }
  \hfill
  \subfigure[Error coverage LR]{%
    \includegraphics[width=0.31\textwidth]{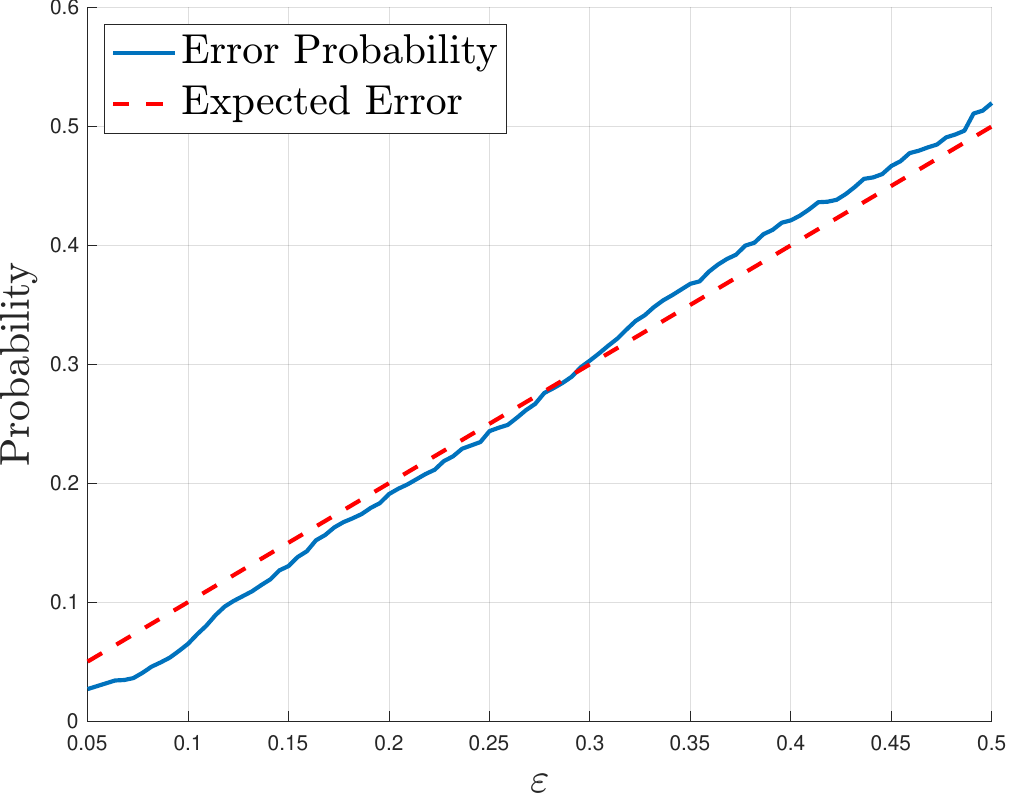} 
    \label{fig:ex6c}
  }
  \caption{Error coverage plot as $\varepsilon$ varies \added{in $[0.05, 0.5]$} for the example classifiers. \added{The probability varies in $[0, 0.5]$ for SVM, $[0.05, 0.55]$ for SVDD and $[0,0.6]$ for LR.}}
  \label{fig:example6}
\end{center}
\end{figure*}  

\section{Conclusions}

Scalable classifiers allow for the development of new techniques to assess safety and robustness in classification problems. With this research we explored the similarities between scalable classifiers and conformal prediction. Through the definition of a score function that naturally derives from the scalable classifier, it is possible to define the concept of conformal safety region, a region that possesses a crucial property known as error coverage, which implies that the probability of observing the wrong label for data points within this region is guaranteed to be no more than a predefined confidence error of $\varepsilon$. Moreover, ongoing studies on the conformal coverage (that is\added{,} the probability of observing the true safe label in the CSR is no less than $1-\varepsilon$) suggest that a mathematical proof for this property is conceivable. The idea is to exploit the results on class-conditional conformal prediction as in \cite{inductive}. \added{In addition, future work will include the possibility of extending the formulation of scalable classifiers, and thus the conformal safety region, to the multi-class and multi-label context.}

The exploration of conformal and error \deleted{coverage}\added{coverages} introduces a novel and meaningful concept that holds great promise for applications in the field of reliable and trustworthy artificial intelligence. It has the potential to enhance the assessment of safety and robustness, contributing to the advancement of AI systems that can be trusted and relied upon in critical applications.

\backmatter







\section*{Declarations}


\begin{itemize}
\item Funding: This work was supported in part by REXASI-PRO H-EU project, call HORIZON-CL4-2021-HUMAN-01-01, Grant agreement ID: 101070028. The work was also supported by Future Artificial Intelligence Research (FAIR) project, Italian Recovery and Resilience Plan (PNRR), Spoke 3 - Resilient AI. Moreover T. Alamo acknowledges support from grant PID2022-142946NA-I00 funded by MCIN/AEI/ 10.13039/501100011033 and by ERDF, A way of making Europe.
\item Conflict of interest/Competing interests (check journal-specific guidelines for which heading to use): There are no either conflicts of interests or competing interests to declare.
\item Ethics approval: Not Applicable.
\item Consent to participate: Not Applicable.
\item Consent for publication: Not Applicable
\item Availability of data and materials: The datasets generated and analysed during the current study are not publicly available due to potentially sensitive data from the CNR-IEIIT's internal network but are available from the corresponding author on reasonable request.
\item Code availability: The authors prefer not to make the codes available yet. It can be requested from the corresponding author upon reasonable request.
\item Authors' contributions:
    \begin{itemize}
        \item Alberto Carlevaro: Methodology, Validation, Investigation, Software, Data curation, Writing.
        \item Teodoro Alamo: Methodology, Validation, Investigation, Writing, Supervision.
        \item Fabrizio Dabbene: Methodology, Validation, Investigation, Writing, Supervision.
        \item Maurizio Mongelli: Methodology, Validation, Investigation, Data curation, Writing, Supervision.
    \end{itemize}
\end{itemize}

\bibliography{sn-bibliography}


\begin{thebibliography}{28}
\ifx \bisbn   \undefined \def \bisbn  #1{ISBN #1}\fi
\ifx \binits  \undefined \def \binits#1{#1}\fi
\ifx \bauthor  \undefined \def \bauthor#1{#1}\fi
\ifx \batitle  \undefined \def \batitle#1{#1}\fi
\ifx \bjtitle  \undefined \def \bjtitle#1{#1}\fi
\ifx \bvolume  \undefined \def \bvolume#1{\textbf{#1}}\fi
\ifx \byear  \undefined \def \byear#1{#1}\fi
\ifx \bissue  \undefined \def \bissue#1{#1}\fi
\ifx \bfpage  \undefined \def \bfpage#1{#1}\fi
\ifx \blpage  \undefined \def \blpage #1{#1}\fi
\ifx \burl  \undefined \def \burl#1{\textsf{#1}}\fi
\ifx \doiurl  \undefined \def \doiurl#1{\url{https://doi.org/#1}}\fi
\ifx \betal  \undefined \def \betal{\textit{et al.}}\fi
\ifx \binstitute  \undefined \def \binstitute#1{#1}\fi
\ifx \binstitutionaled  \undefined \def \binstitutionaled#1{#1}\fi
\ifx \bctitle  \undefined \def \bctitle#1{#1}\fi
\ifx \beditor  \undefined \def \beditor#1{#1}\fi
\ifx \bpublisher  \undefined \def \bpublisher#1{#1}\fi
\ifx \bbtitle  \undefined \def \bbtitle#1{#1}\fi
\ifx \bedition  \undefined \def \bedition#1{#1}\fi
\ifx \bseriesno  \undefined \def \bseriesno#1{#1}\fi
\ifx \blocation  \undefined \def \blocation#1{#1}\fi
\ifx \bsertitle  \undefined \def \bsertitle#1{#1}\fi
\ifx \bsnm \undefined \def \bsnm#1{#1}\fi
\ifx \bsuffix \undefined \def \bsuffix#1{#1}\fi
\ifx \bparticle \undefined \def \bparticle#1{#1}\fi
\ifx \barticle \undefined \def \barticle#1{#1}\fi
\bibcommenthead
\ifx \bconfdate \undefined \def \bconfdate #1{#1}\fi
\ifx \botherref \undefined \def \botherref #1{#1}\fi
\ifx \url \undefined \def \url#1{\textsf{#1}}\fi
\ifx \bchapter \undefined \def \bchapter#1{#1}\fi
\ifx \bbook \undefined \def \bbook#1{#1}\fi
\ifx \bcomment \undefined \def \bcomment#1{#1}\fi
\ifx \oauthor \undefined \def \oauthor#1{#1}\fi
\ifx \citeauthoryear \undefined \def \citeauthoryear#1{#1}\fi
\ifx \endbibitem  \undefined \def \endbibitem {}\fi
\ifx \bconflocation  \undefined \def \bconflocation#1{#1}\fi
\ifx \arxivurl  \undefined \def \arxivurl#1{\textsf{#1}}\fi
\csname PreBibitemsHook\endcsname

\bibitem[\protect\citeauthoryear{Shafer and Vovk}{2008}]{JMLR:v9:shafer08a}
\begin{bbook}
\bauthor{\bsnm{Shafer}, \binits{G.}},
\bauthor{\bsnm{Vovk}, \binits{V.}}:
\bbtitle{A Tutorial on Conformal Prediction}
vol. \bseriesno{9},
pp. \bfpage{371}--\blpage{421}
(\byear{2008}).
\burl{http://jmlr.org/papers/v9/shafer08a.html}
\end{bbook}
\endbibitem

\bibitem[\protect\citeauthoryear{Vovk et~al.}{2005}]{vovk2005algorithmic}
\begin{bbook}
\bauthor{\bsnm{Vovk}, \binits{V.}},
\bauthor{\bsnm{Gammerman}, \binits{A.}},
\bauthor{\bsnm{Shafer}, \binits{G.}}:
\bbtitle{Algorithmic Learning in a Random World}.
\bpublisher{Springer},
\blocation{Berlin, Heidelberg}
(\byear{2005})
\end{bbook}
\endbibitem

\bibitem[\protect\citeauthoryear{Vovk et~al.}{2017}]{vovk2017nonparametric}
\begin{bchapter}
\bauthor{\bsnm{Vovk}, \binits{V.}},
\bauthor{\bsnm{Shen}, \binits{J.}},
\bauthor{\bsnm{Manokhin}, \binits{V.}},
\bauthor{\bsnm{Xie}, \binits{M.}}:
\bctitle{Nonparametric predictive distributions based on conformal prediction}.
In: \bbtitle{Conformal and Probabilistic Prediction and Applications},
pp. \bfpage{82}--\blpage{102}
(\byear{2017})
\end{bchapter}
\endbibitem

\bibitem[\protect\citeauthoryear{Vovk et~al.}{2022}]{vovk2022probabilistic}
\begin{bbook}
\bauthor{\bsnm{Vovk}, \binits{V.}},
\bauthor{\bsnm{Gammerman}, \binits{A.}},
\bauthor{\bsnm{Shafer}, \binits{G.}}:
\bbtitle{Probabilistic Classification: Venn Predictors},
pp. \bfpage{157}--\blpage{179}.
\bpublisher{Springer},
\blocation{Cham}
(\byear{2022}).
\doiurl{10.1007/978-3-031-06649-8_6}
\end{bbook}
\endbibitem

\bibitem[\protect\citeauthoryear{Angelopoulos and Bates}{2023}]{gentleIntro}
\begin{barticle}
\bauthor{\bsnm{Angelopoulos}, \binits{A.N.}},
\bauthor{\bsnm{Bates}, \binits{S.}}:
\batitle{Conformal prediction: A gentle introduction}.
\bjtitle{Foundations and Trends® in Machine Learning}
\bvolume{16}(\bissue{4}),
\bfpage{494}--\blpage{591}
(\byear{2023})
\doiurl{10.1561/2200000101}
\end{barticle}
\endbibitem

\bibitem[\protect\citeauthoryear{Fontana et~al.}{2023}]{fontana2020conformal}
\begin{barticle}
\bauthor{\bsnm{Fontana}, \binits{M.}},
\bauthor{\bsnm{Zeni}, \binits{G.}},
\bauthor{\bsnm{Vantini}, \binits{S.}}:
\batitle{{Conformal prediction: A unified review of theory and new challenges}}.
\bjtitle{Bernoulli}
\bvolume{29}(\bissue{1}),
\bfpage{1}--\blpage{23}
(\byear{2023})
\doiurl{10.3150/21-BEJ1447}
\end{barticle}
\endbibitem

\bibitem[\protect\citeauthoryear{Toccaceli}{2022}]{toccaceli2022introduction}
\begin{barticle}
\bauthor{\bsnm{Toccaceli}, \binits{P.}}:
\batitle{Introduction to conformal predictors}.
\bjtitle{Pattern Recognition}
\bvolume{124},
\bfpage{108507}
(\byear{2022})
\doiurl{10.1016/j.patcog.2021.108507}
\end{barticle}
\endbibitem

\bibitem[\protect\citeauthoryear{Forreryd et~al.}{2018}]{forreryd2018predicting}
\begin{barticle}
\bauthor{\bsnm{Forreryd}, \binits{A.}},
\bauthor{\bsnm{Norinder}, \binits{U.}},
\bauthor{\bsnm{Lindberg}, \binits{T.}},
\bauthor{\bsnm{Lindstedt}, \binits{M.}}:
\batitle{Predicting skin sensitizers with confidence — using conformal prediction to determine applicability domain of gard}.
\bjtitle{Toxicology in Vitro}
\bvolume{48},
\bfpage{179}--\blpage{187}
(\byear{2018})
\doiurl{10.1016/j.tiv.2018.01.021}
\end{barticle}
\endbibitem

\bibitem[\protect\citeauthoryear{Balasubramanian et~al.}{2009}]{balasubramanian2009support}
\begin{bchapter}
\bauthor{\bsnm{Balasubramanian}, \binits{V.N.}},
\bauthor{\bsnm{Gouripeddi}, \binits{R.}},
\bauthor{\bsnm{Panchanathan}, \binits{S.}},
\bauthor{\bsnm{Vermillion}, \binits{J.}},
\bauthor{\bsnm{Bhaskaran}, \binits{A.}},
\bauthor{\bsnm{Siegel}, \binits{R.}}:
\bctitle{Support vector machine based conformal predictors for risk of complications following a coronary drug eluting stent procedure}.
In: \bbtitle{2009 36th Annual Computers in Cardiology Conference (CinC)},
pp. \bfpage{5}--\blpage{8}
(\byear{2009}).
\bcomment{IEEE}
\end{bchapter}
\endbibitem

\bibitem[\protect\citeauthoryear{Narteni et~al.}{2023}]{narteni2023confiderai}
\begin{bchapter}
\bauthor{\bsnm{Narteni}, \binits{S.}},
\bauthor{\bsnm{Carlevaro}, \binits{A.}},
\bauthor{\bsnm{Dabbene}, \binits{F.}},
\bauthor{\bsnm{Muselli}, \binits{M.}},
\bauthor{\bsnm{Mongelli}, \binits{M.}}:
\bctitle{Confiderai: Conformal interpretable-by-design score function for explainable and reliable artificial intelligence}.
In: \bbtitle{Conformal and Probabilistic Prediction with Applications},
pp. \bfpage{485}--\blpage{487}
(\byear{2023})
\end{bchapter}
\endbibitem

\bibitem[\protect\citeauthoryear{Angelopoulos et~al.}{2020}]{angelopoulos-sets}
\begin{bchapter}
\bauthor{\bsnm{Angelopoulos}, \binits{A.N.}},
\bauthor{\bsnm{Bates}, \binits{S.}},
\bauthor{\bsnm{Jordan}, \binits{M.}},
\bauthor{\bsnm{Malik}, \binits{J.}}:
\bctitle{Uncertainty sets for image classifiers using conformal prediction}.
In: \bbtitle{International Conference on Learning Representations}
(\byear{2020})
\end{bchapter}
\endbibitem

\bibitem[\protect\citeauthoryear{Park et~al.}{2019}]{park2019pac}
\begin{bchapter}
\bauthor{\bsnm{Park}, \binits{S.}},
\bauthor{\bsnm{Bastani}, \binits{O.}},
\bauthor{\bsnm{Matni}, \binits{N.}},
\bauthor{\bsnm{Lee}, \binits{I.}}:
\bctitle{Pac confidence sets for deep neural networks via calibrated prediction}.
In: \bbtitle{International Conference on Learning Representations}
(\byear{2019})
\end{bchapter}
\endbibitem

\bibitem[\protect\citeauthoryear{And{\'e}ol et~al.}{2024}]{andeol2023conformal}
\begin{botherref}
\oauthor{\bsnm{And{\'e}ol}, \binits{L.}},
\oauthor{\bsnm{Fel}, \binits{T.}},
\oauthor{\bsnm{De~Grancey}, \binits{F.}},
\oauthor{\bsnm{Mossina}, \binits{L.}}:
Conformal prediction for trustworthy detection of railway signals.
AI and Ethics,
1--5
(2024)
\end{botherref}
\endbibitem

\bibitem[\protect\citeauthoryear{Carlevaro et~al.}{2023}]{carlevaro2023probabilistic}
\begin{botherref}
\oauthor{\bsnm{Carlevaro}, \binits{A.}},
\oauthor{\bsnm{Alamo}, \binits{T.}},
\oauthor{\bsnm{Dabbene}, \binits{F.}},
\oauthor{\bsnm{Mongelli}, \binits{M.}}:
Probabilistic safety regions via finite families of scalable classifiers
(2023)
{\href{https://arxiv.org/abs/2309.04627}{{arXiv:2309.04627}}}
{[stat.ML]}
\end{botherref}
\endbibitem

\bibitem[\protect\citeauthoryear{Chzhen et~al.}{2021}]{chzhen2021set}
\begin{botherref}
\oauthor{\bsnm{Chzhen}, \binits{E.}},
\oauthor{\bsnm{Denis}, \binits{C.}},
\oauthor{\bsnm{Hebiri}, \binits{M.}},
\oauthor{\bsnm{Lorieul}, \binits{T.}}:
Set-valued classification--overview via a unified framework
(2021)
{\href{https://arxiv.org/abs/2102.12318}{{arXiv:2102.12318}}}
{[stat.ML]}
\end{botherref}
\endbibitem

\bibitem[\protect\citeauthoryear{Lenatti et~al.}{2022}]{10.1371/journal.pone.0272825}
\begin{barticle}
\bauthor{\bsnm{Lenatti}, \binits{M.}},
\bauthor{\bsnm{Carlevaro}, \binits{A.}},
\bauthor{\bsnm{Guergachi}, \binits{A.}},
\bauthor{\bsnm{Keshavjee}, \binits{K.}},
\bauthor{\bsnm{Mongelli}, \binits{M.}},
\bauthor{\bsnm{Paglialonga}, \binits{A.}}:
\batitle{A novel method to derive personalized minimum viable recommendations for type 2 diabetes prevention based on counterfactual explanations}.
\bjtitle{PLOS ONE}
\bvolume{17}(\bissue{11}),
\bfpage{1}--\blpage{24}
(\byear{2022})
\doiurl{10.1371/journal.pone.0272825}
\end{barticle}
\endbibitem

\bibitem[\protect\citeauthoryear{Carlevaro et~al.}{2022}]{9787552}
\begin{barticle}
\bauthor{\bsnm{Carlevaro}, \binits{A.}},
\bauthor{\bsnm{Lenatti}, \binits{M.}},
\bauthor{\bsnm{Paglialonga}, \binits{A.}},
\bauthor{\bsnm{Mongelli}, \binits{M.}}:
\batitle{Counterfactual building and evaluation via explainable support vector data description}.
\bjtitle{IEEE Access}
\bvolume{10},
\bfpage{60849}--\blpage{60861}
(\byear{2022})
\doiurl{10.1109/ACCESS.2022.3180026}
\end{barticle}
\endbibitem

\bibitem[\protect\citeauthoryear{Vovk et~al.}{1999}]{VovkGuarantee}
\begin{bchapter}
\bauthor{\bsnm{Vovk}, \binits{V.}},
\bauthor{\bsnm{Gammerman}, \binits{A.}},
\bauthor{\bsnm{Saunders}, \binits{C.}}:
\bctitle{Machine-learning applications of algorithmic randomness}.
In: \bbtitle{Proceedings of the Sixteenth International Conference on Machine Learning}.
\bsertitle{ICML '99},
pp. \bfpage{444}--\blpage{453}.
\bpublisher{Morgan Kaufmann Publishers Inc.},
\blocation{San Francisco, CA, USA}
(\byear{1999})
\end{bchapter}
\endbibitem

\bibitem[\protect\citeauthoryear{H{\"u}llermeier and Waegeman}{2021}]{hullermeier2021aleatoric}
\begin{barticle}
\bauthor{\bsnm{H{\"u}llermeier}, \binits{E.}},
\bauthor{\bsnm{Waegeman}, \binits{W.}}:
\batitle{Aleatoric and epistemic uncertainty in machine learning: An introduction to concepts and methods}.
\bjtitle{Machine Learning}
\bvolume{110},
\bfpage{457}--\blpage{506}
(\byear{2021})
\doiurl{10.1007/s10994-021-05946-3}
\end{barticle}
\endbibitem

\bibitem[\protect\citeauthoryear{Sale et~al.}{2023}]{sale2023volume}
\begin{botherref}
\oauthor{\bsnm{Sale}, \binits{Y.}},
\oauthor{\bsnm{Caprio}, \binits{M.}},
\oauthor{\bsnm{Hüllermeier}, \binits{E.}}:
Is the volume of a credal set a good measure for epistemic uncertainty?
(2023)
{\href{https://arxiv.org/abs/2306.09586}{{arXiv:2306.09586}}}
{[cs.LG]}
\end{botherref}
\endbibitem

\bibitem[\protect\citeauthoryear{Abell{\'a}n et~al.}{2006}]{abellan2006disaggregated}
\begin{barticle}
\bauthor{\bsnm{Abell{\'a}n}, \binits{J.}},
\bauthor{\bsnm{Klir}, \binits{G.J.}},
\bauthor{\bsnm{Moral}, \binits{S.}}:
\batitle{Disaggregated total uncertainty measure for credal sets}.
\bjtitle{International Journal of General Systems}
\bvolume{35}(\bissue{1}),
\bfpage{29}--\blpage{44}
(\byear{2006})
\end{barticle}
\endbibitem

\bibitem[\protect\citeauthoryear{Sale et~al.}{2023}]{sale2023second}
\begin{botherref}
\oauthor{\bsnm{Sale}, \binits{Y.}},
\oauthor{\bsnm{Bengs}, \binits{V.}},
\oauthor{\bsnm{Caprio}, \binits{M.}},
\oauthor{\bsnm{H{\"u}llermeier}, \binits{E.}}:
Second-order uncertainty quantification: A distance-based approach
(2023)
{\href{https://arxiv.org/abs/2312.00995}{{arXiv:2312.00995}}}
{[cs.LG]}
\end{botherref}
\endbibitem

\bibitem[\protect\citeauthoryear{Valiant}{2013}]{valiant2013probably}
\begin{bbook}
\bauthor{\bsnm{Valiant}, \binits{L.}}:
\bbtitle{Probably Approximately Correct: Nature's Algorithms for Learning and Prospering in a Complex World}.
\bpublisher{Basic Books, Inc.},
\blocation{USA}
(\byear{2013})
\end{bbook}
\endbibitem

\bibitem[\protect\citeauthoryear{Park et~al.}{2022}]{park2022pac}
\begin{barticle}
\bauthor{\bsnm{Park}, \binits{S.}},
\bauthor{\bsnm{Dobriban}, \binits{E.}},
\bauthor{\bsnm{Lee}, \binits{I.}},
\bauthor{\bsnm{Bastani}, \binits{O.}}:
\batitle{\protect{PAC} prediction sets for meta-learning}.
\bjtitle{Advances in Neural Information Processing Systems}
\bvolume{35},
\bfpage{37920}--\blpage{37931}
(\byear{2022})
\end{barticle}
\endbibitem

\bibitem[\protect\citeauthoryear{Vovk}{2012}]{inductive}
\begin{bchapter}
\bauthor{\bsnm{Vovk}, \binits{V.}}:
\bctitle{Conditional validity of inductive conformal predictors}.
In: \beditor{\bsnm{Hoi}, \binits{S.C.H.}},
\beditor{\bsnm{Buntine}, \binits{W.}} (eds.)
\bbtitle{Proceedings of the Asian Conference on Machine Learning}.
\bsertitle{Proceedings of Machine Learning Research},
vol. \bseriesno{25},
pp. \bfpage{475}--\blpage{490}.
\bpublisher{PMLR},
\blocation{Singapore Management University, Singapore}
(\byear{2012}).
\burl{https://proceedings.mlr.press/v25/vovk12.html}
\end{bchapter}
\endbibitem

\bibitem[\protect\citeauthoryear{Aiello et~al.}{2015}]{aiello2015dns}
\begin{barticle}
\bauthor{\bsnm{Aiello}, \binits{M.}},
\bauthor{\bsnm{Mongelli}, \binits{M.}},
\bauthor{\bsnm{Papaleo}, \binits{G.}}:
\batitle{\protect{DNS} tunneling detection through statistical fingerprints of protocol messages and machine learning}.
\bjtitle{International Journal of Communication Systems}
\bvolume{28}(\bissue{14}),
\bfpage{1987}--\blpage{2002}
(\byear{2015})
\end{barticle}
\endbibitem

\bibitem[\protect\citeauthoryear{Carlevaro and Mongelli}{2021}]{9594676}
\begin{barticle}
\bauthor{\bsnm{Carlevaro}, \binits{A.}},
\bauthor{\bsnm{Mongelli}, \binits{M.}}:
\batitle{A new {SVDD} approach to reliable and e{X}plainable {AI}}.
\bjtitle{IEEE Intelligent Systems}
(\byear{2021})
\doiurl{10.1109/MIS.2021.3123669}
\end{barticle}
\endbibitem

\bibitem[\protect\citeauthoryear{Vaccari et~al.}{2022}]{Ivan}
\begin{barticle}
\bauthor{\bsnm{Vaccari}, \binits{I.}},
\bauthor{\bsnm{Carlevaro}, \binits{A.}},
\bauthor{\bsnm{Narteni}, \binits{S.}},
\bauthor{\bsnm{Cambiaso}, \binits{E.}},
\bauthor{\bsnm{Mongelli}, \binits{M.}}:
\batitle{\protect{eXplainable and Reliable Against Adversarial Machine Learning in Data Analytics}}.
\bjtitle{IEEE Access}
\bvolume{10},
\bfpage{83949}--\blpage{83970}
(\byear{2022})
\doiurl{10.1109/ACCESS.2022.3197299}
\end{barticle}
\endbibitem

\end{thebibliography}

\end{document}